\documentclass[sigconf,nonatbib]{acmart} 
\usepackage[utf8]{inputenc}
\usepackage{balance}
\usepackage[T1]{fontenc}

\copyrightyear{2021}
\acmYear{2021}
\setcopyright{rightsretained}
\acmConference[CIKM '21] {Proceedings of the 30th ACM International Conference on Information and Knowledge Management}{November 1--5, 2021}{Virtual Event, Australia.}
\acmBooktitle{Proceedings of the 30th ACM Int'l Conf. on Information and Knowledge Management (CIKM '21), November 1--5, 2021, Virtual Event, Australia}
\acmPrice{}
\acmISBN{978-1-4503-8446-9/21/11}
\acmDOI{10.1145/3459637.3481914}

\usepackage{algorithm,algpseudocode,subcaption,graphicx,tikz,hyperref,wrapfig}
\usepackage[english]{babel}

\graphicspath{{../plots/}}
\usepackage{amsfonts, dsfont, amsthm, amsmath,graphicx,xcolor,algorithm}

\newcommand{\ignore}[1]{}

\definecolor{forestgreen}{rgb}{0.0, 0.27, 0.13}

\usepackage{semantic}
\mathlig{==}{\equiv}
\mathlig{=.}{\doteq}
\mathlig{:=}{\triangleq}
\mathlig{<<}{\ll}
\mathlig{>>}{\gg}
\mathlig{<>}{\neq}
\mathlig{<=}{\leq}
\mathlig{>=}{\geq}
\mathlig{<==}{\Leftarrow}
\mathlig{==>}{\Rightarrow}
\mathlig{<=>}{\Leftrightarrow}
\mathlig{<==>}{\iff}
\mathlig{<--}{\leftarrow}
\mathlig{-->}{\rightarrow}
\mathlig{<->}{\leftrightarrow}
\mathlig{+-}{\pm}
\mathlig{-+}{\mp}
\mathlig{...}{\dots}
\mathlig{!=}{\stackrel{!}{=}}

\DeclareMathOperator*{\argmax}{\arg\!\max}

\newif\ifshowanswer    

\newcommand{\isitthree}[1]
{
  \ifnum#1=3
    number #1 is 3
  \else
    number #1 is not 3
  \fi
}

\newcommand{\be}{\begin{equation}}
\newcommand{\ee}{\end{equation}}

\newcommand\R{{\mathbb{R}}}

\renewcommand\b{y}

\newcommand\CS{{\mathcal S}}

\renewcommand{\b}[1]{\boldsymbol{#1}}
\newcommand\numberthis{\addtocounter{equation}{1}\tag{\theequation}}
\newtheorem{theorem}{Theorem}

\newtheorem{lemma}{Lemma}

\newcommand{\opt}{^\star}

\newcommand{\lt}{{\ell_\text{t}}}

\begin{document}
\fancyhead{}
\title[Enabling Efficiency-Precision Trade-offs for Label Trees in Extreme Classification]{Enabling Efficiency-Precision Trade-offs \\ for Label Trees in Extreme Classification}

\author{Tavor Z. Baharav}
\affiliation{ \institution{Stanford University}}\email{tavorb@stanford.edu}
\author{Daniel L. Jiang}\authornote{Work done while at Amazon.}
\affiliation{University of Washington}\email{danji@cs.washington.edu}
\author{Kedarnath Kolluri}
\affiliation{Amazon}\email{kkolluri@amazon.com}
\author{Sujay Sanghavi}
\affiliation{Amazon}
\affiliation{University of Texas at Austin} \email{sujayrs@amazon.com}
\author{Inderjit S. Dhillon}
\affiliation{Amazon}
\affiliation{University of Texas at Austin} \email{isd@amazon.com}

\renewcommand{\shortauthors}{Baharav, Jiang, Kolluri, Sanghavi, and Dhillon}

\begin{abstract}
Extreme multi-label classification (XMC) aims to learn a model that can tag data points with a subset of relevant labels from an extremely large label set.
Real world e-commerce applications like personalized recommendations and product advertising can be formulated as XMC problems, where the objective is to predict for a user a small subset of items from a catalog of several million products.
For such applications, a common approach is to organize these labels into a tree, enabling training and inference times that are logarithmic in the number of labels \cite{parabel}.
While training a model once a label tree is available is well studied, designing the structure of the tree is a difficult task that is not yet well understood, and can dramatically impact both model latency and statistical performance.
Existing approaches to tree construction either optimize exclusively for statistical performance or optimize exclusively for latency.
We propose an efficient information theory inspired algorithm to construct intermediate operating points that 
trade off between the benefits of both, which was not previously possible.
We corroborate our theoretical analysis with numerical results, showing that on the Wiki-500K \cite{Bhatia16_XMCRepo} benchmark dataset our method can reduce a proxy for expected latency by up to 28\% while maintaining the same accuracy as Parabel \cite{parabel}.
On several datasets derived from e-commerce customer logs, our modified label tree is able to improve this expected latency metric by up to 20\% while maintaining the same accuracy. Finally, we discuss challenges in realizing these latency improvements in deployed models.
\end{abstract}



\maketitle
\section{Introduction}
With the ever increasing size of datasets, a new paradigm of classification problems has emerged in machine learning.
In the setting of multi-label classification, the goal is to learn a model that can tag data points with a subset of relevant labels from a given label set.
One common approach to this problem is 1-vs-All classification, where a separate classifier is learned for each label and all classifiers are evaluated at inference time, resulting in training and inference costs linear in the number of labels \cite{Bhatia16_XMCRepo}.
As the sizes of industrial datasets grow, the number of possible labels in these applications can easily reach hundreds of thousands or even millions, making the linear complexity of 1-vs-All methods prohibitive.
This motivates the paradigm of extreme multi-label classification (XMC), where the number of labels,  the number of points, and their dimensionality, are all extremely large \cite{agrawal2013multi}.
Many modern large-scale industrial applications are routinely modeled as XMC problems, such as webpage annotation \cite{partalas2015lshtc}, text classification \cite{joulin2016bag_fasttext,mikolov2013distributed,you2018attentionxml}, dynamic search advertisement \cite{parabel}, and text similarity search \cite{chang2019xBert}.

To overcome the challenge of extremely large label spaces, many state of the art methods first organize the labels hierarchically into a search tree.
These tree-based methods then learn a separate classifier for each internal node in the tree to predict whether a label relevant to the given context appears in the subtree rooted at that node.
By utilizing the greedy traversal algorithm of beam search, these methods only evaluate the classifiers along a constant number of paths in the tree, resulting in efficient inference (with costs logarithmic in the number of labels for a balanced binary label tree) \cite{parabel,yu2020pecos}.
Thus far, the design of this search tree has fallen into one of two categories: \textit{similarity-based} \cite{parabel} or \textit{coding-theoretic} \cite{joulin2016bag_fasttext}.

Similarity-based methods construct trees that optimize purely for statistical performance.
These methods ensure that similar labels are placed close together in the tree.
This helps yield meaningful label partitions at each internal node of the tree, and so the tree's internal classifiers can be expected to achieve high accuracy.
An example of a model that structures its label tree this way is Parabel \cite{parabel}, which constructs its tree by recursively applying balanced 2-means clustering to the set of label embeddings. 

Such methods are commonly used in XMC solutions for e-commerce applications.
One canonical example is product advertising, where the customer query is used as the input, and the catalog of sponsored products is the set of labels.
Similarly, several forms of personalized recommendations can be modeled in the XMC framework.
For example, the widget "Products related to this item" on an Amazon product page could be an XMC problem with the current product's information as the context and the products in Amazon's catalog as labels.
In such applications, it is as important to retrieve results quickly as it is to have results that are relevant for the customer.
Improvements in training or prediction efficiency of XMC models serving e-commerce applications can also yield savings in infrastructure costs, which must be considered alongside relevance.

\setlength{\columnsep}{7pt} 
\begin{wrapfigure}{R}{0.5\columnwidth}
    \centering
    \vspace{-.35cm} 
	\includegraphics[width=.5\columnwidth, trim=.3cm .5cm .9cm 1.4cm, clip]{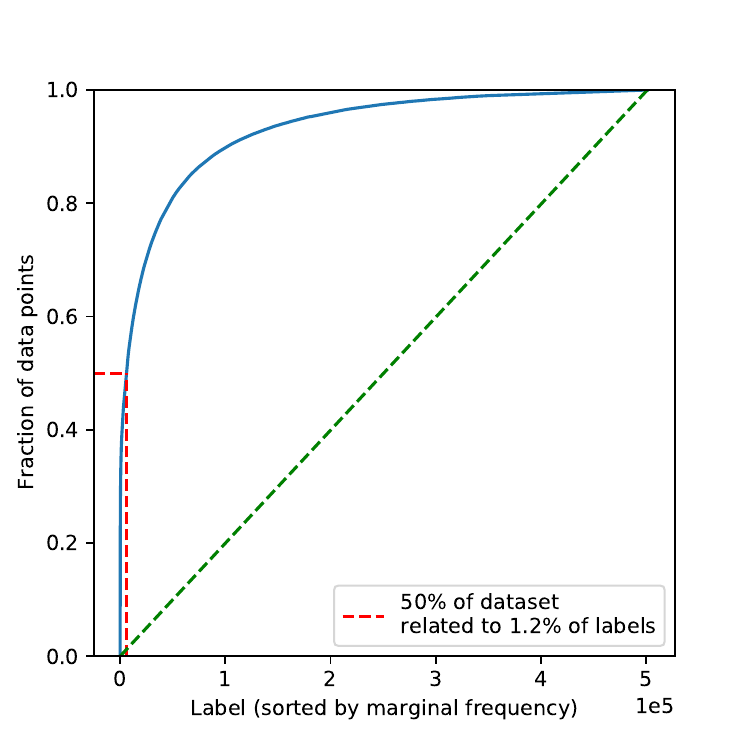}
    \vspace{-.75cm}
    \caption{Label frequency imbalance on Wiki-500K.}
    \vspace{-.45cm}
    \label{fig:wiki500k_labelFreqs}
\end{wrapfigure}
Similarity-based trees, however, do not explicitly aim to address infrastructure costs.
While tree-based methods scale the computational complexity from linear in the number of labels to logarithmic, further improvements have been achieved when some labels are more frequently matched to the context than others, causing an imbalance in label frequencies \cite{mikolov2013distributed}. 
Not surprisingly, such an imbalance occurs in many XMC datasets, where the label frequencies are often well approximated by a power-law distribution \cite{khandagale2020bonsai}.
For example, in Wiki-500K, a common benchmark XMC dataset, the most frequent 1\% of the labels can provide at least one relevant label for 50\% of the dataset (Figure \ref{fig:wiki500k_labelFreqs}, see Appendix \ref{app:additionalSims} for more details).
This phenomenon creates opportunity to further optimize training and inference costs by placing frequently occurring labels higher in the tree.
Existing similarity-based methods do not perform this optimization, as they consider only the label feature space when clustering, placing every label at the same depth.

On the other end of the spectrum, coding-theoretic trees optimize purely for the expected depth of the returned labels, ignoring the label feature space entirely.
One model which utilizes such a tree is fastText \cite{joulin2016bag_fasttext}, which applies Huffman coding to the label frequencies to construct a label tree for a hierarchical softmax.
While such models yield efficient training and inference, observed speed-ups appear to be at the cost of accuracy \cite{mikolov2013distributed}.
For clarity and theoretical grounding, in this work we analyze the tree metric of expected depth (expected latency) as a proxy for training and prediction computational costs \cite{mikolov2013distributed}.

On the surface, these similarity-based and coding-theoretic methods are at an apparent impasse.
The recent work of \cite{busa2019computational} on probabilistic label trees (PLTs), a formalization of the label trees previously discussed, posed a fundamental question that we tackle in this paper: 
\textit{``to find a tree structure that results in a PLT with a low training and prediction computational costs as well as low statistical error seems to be a very challenging problem, not well-understood yet''}.
In this paper, we design a scheme that can interpolate between similarity-based and coding-theoretic trees, allowing one to smoothly trade off between statistical performance and expected latency.
Our contributions in this work are twofold:
\begin{enumerate}
	\item Provide a unified framework to study probabilistic label trees for datasets with both frequencies and similarity measures.
	\item Design an objective and algorithm for constructing PLTs with a tunable hyperparameter to interpolate between the computationally efficient and statistically efficient solutions.
\end{enumerate}

\begin{wrapfigure}{R}{0.5\columnwidth}
	\vspace{-.4cm}
	\includegraphics[width=.5\columnwidth]{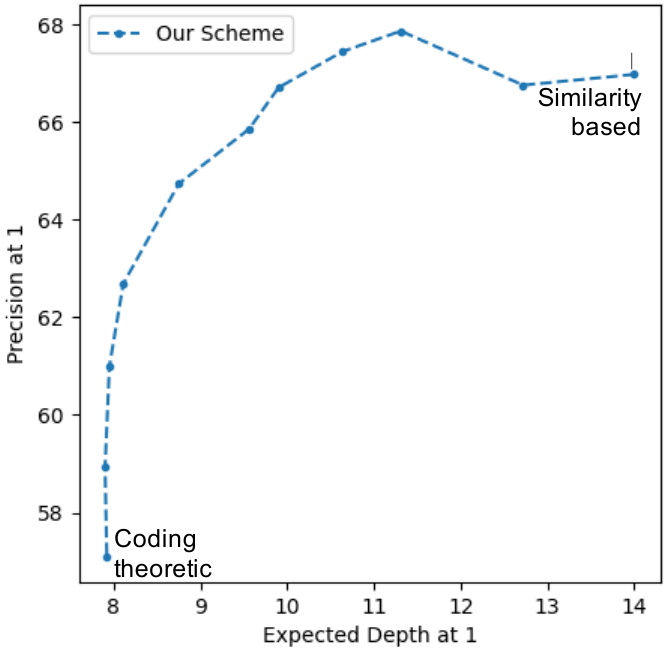}
	\vspace{-.8cm}
	\caption{New operating points offered by our scheme, shown on Wiki-500K dataset.}
	\label{fig:newOperatingPoints}
	\vspace{-.5cm}
\end{wrapfigure}
Our solution trades off between label relevance and expected
latency in a principled manner, and allows for operating points beyond simply these two extremes, as shown in Figure \ref{fig:newOperatingPoints}. 
While there is a general Pareto-style trade-off between these two metrics, we observe a surprising phenomenon across datasets; we are able to marginally \textit{improve} statistical performance while reducing expected depth for some target operating points.
On the Wiki-500K benchmark dataset, we see in Figure \ref{fig:newOperatingPoints} that our method can reduce expected depth by 28\% while maintaining the same accuracy.
We also show that when we replace the coding-theoretic Huffman tree in fastText with a label tree constructed using methods described in this manuscript, we improve model accuracy on the Wiki-500K dataset by 40.6\% while increasing expected depth by only 6.5\%.
We also show that incorporating modified label trees in Parabel improves the average depth traversed by up to 20\% on several XMC datasets derived from e-commerce customer logs with no reduction in statistical performance.

We proceed by discussing related work and PLT background in Section \ref{sec:relatedWorks}.
In Section \ref{sec:clusteringSchemes} we discuss the two extremes of the spectrum that are present in the existing literature.
In Section \ref{sec:interpolatedPLT} we present our novel algorithm that interpolates between these two endpoints.
In Section \ref{sec:simulations} we provide numerical results showing the improvement afforded by our scheme on e-commerce customer logs as well as public datasets, and provide theoretical backing to corroborate our experimental results.
We conclude in Section \ref{sec:conc}.

\section{Related Work} \label{sec:relatedWorks}

XMC has seen a surge of work in recent years due to the rapidly increasing size of datasets.
As previously discussed, one of the most accurate approaches to the XMC problem is to perform 1-vs-All classification \cite{Bhatia16_XMCRepo}.
1-vs-All models such as DiSMEC \cite{dismec}, PD-Sparse \cite{pdsparse}, PPDSparse \cite{ppdsparse}, and XML-CNN \cite{xmlcnn} learn a separate linear classifier for each label and evaluate all classifiers at inference time.
While such methods often have good statistical performance, their inference times are necessarily linear in the number of labels.

Another class of approaches utilizes hashing to reduce an XMC problem down to a few small classification problems, saving storage by obviating the requirement of storing a tree. These methods have similar statistical performance to tree-based ones, but unfortunately require a fixed amount of time for each query, the same as a balanced tree, and cannot easily incorporate frequency information to improve expected latency \cite{vijayanarasimhan2014deep,shrivastava2014asymmetric,huang2018extreme}.

\subsection{Probabilistic Label Trees}
With the increasing size of datasets, a linear dependency on the number of labels is no longer feasible in many e-commerce applications.
A tree-based solution was first proposed in \cite{morin2005hierarchical}, which sped up the training time for a softmax by organizing the labels into a tree, creating a hierarchical softmax.
This tree, with labels as its leaves, operates by routing inputs down the tree to a predicted relevant label (or more generally a set of labels).
In this setting each internal node of this tree contains two classifiers, which estimate the probability that a data point has a relevant label contained in the left subtree or the right subtree.
To ensure that the classifiers perform well, it is desirable to have a tree that obeys the similarity structure of the data; that is, labels that are similar to each other should be close together in this tree, whereas dissimilar labels should be far from each other.
In the initial work of \cite{morin2005hierarchical}, a WordNet based clustering was hand-designed.
However, it was later realized that these trees could be learned from data \cite{mnih2009scalable_parabel_esque}.
For this NLP task, the label embeddings for the softmax were hierarchically clustered into a tree via a divisive algorithm, where at every step the set of leaves is partitioned in half by fitting two Gaussians to their embeddings.
Several recent works have focused on learning improved clusterings for higher accuracy trees \cite{prabhu2014fastxml,pmlr-v48-jasinska16,prabhu2018extreme, yu2020pecos, khandagale2020bonsai}.
Of particular interest to this work {and e-commerce applications} is Parabel \cite{parabel}, which recursively uses balanced spherical 2-means clustering to create a balanced binary tree on the labels.
More generally, these label trees can be analyzed in the context of \textit{Probabilistic Label Trees} (PLTs).
We direct the interested reader to \cite{busa2019computational} for a formalization of PLTs and additional background.
\subsection{Prefix-free coding}
One shortcoming of existing PLT construction methods is that they are optimizing solely for the statistical accuracy of the tree; they do not optimize the inference time latency beyond that of a balanced binary tree.
Starting in \cite{morin2005hierarchical} with the hierarchical softmax, it was observed that this PLT construction could be viewed through the lens of prefix-free coding, as each path to a label in the tree can be seen as a binary string comprising the codeword for that given label.
Information theory, specifically source coding, deals with the problem of representing a set of items in a minimum redundancy manner. 
Dating back to \cite{fano1949transmission} researchers have worked on constructing optimal prefix-free codes; that is, constructing (binary) trees on a set of items such that the expected depth of an item selected randomly from a known distribution over these items would be minimized.
The most well known such method was introduced by Huffman \cite{huffman1952method}, which proceeds agglomeratively by iteratively merging the two least frequent items.

Working on hierarchical softmax, \cite{mikolov2013distributed,mikolov2013efficient} observed that Huffman codes could be utilized to construct PLTs.
While similarity-based PLTs have good statistical performance, they have suboptimal latency because for common inputs one needs to traverse the entire depth of the tree every time, which is expensive.
They noted that more frequent words should be higher in this tree to minimize expected compute time, and used Huffman coding to construct the PLT.
This guaranteed optimal expected depth, minimizing training time.
While faster, this approach yields worse statistical performance than using a word embedding-based label tree for the hierarchical softmax \cite{mnih2009scalable_parabel_esque}. 
Subsequent work like fastText \cite{joulin2016bag_fasttext} also used a Huffman code to construct the PLT.

\subsection{Best of both worlds}

While statistical efficiency and computational efficiency are both desirable features on their own, in practice we want a solution that performs well in both of these metrics.
Recently, the idea of combining these two approaches was considered in \cite{yang2017optimize_semhuff}.
In their work, a Huffman code is generated on the words, and the output tree is post processed by rearranging the leaf nodes within a level to optimize for similarity of the word embeddings.
This approach is often suboptimal however, as it first constructs a coding-theoretic PLT and then performs slight modifications to improve statistical performance, as opposed to optimizing for the two objectives simultaneously.
One primary difficulty in this arises from the divisive nature of balanced 2-means clustering and other similarity-based tree construction methods, as compared to the agglomerative bottom up nature of Huffman and Shannon-Fano coding  \cite{busa2019computational}.

A motivating observation is that while previously utilized codes like Huffman and Shannon-Fano are agglomeratively constructed, there exist prefix-free binary codes that can be divisively constructed.
In particular, one of the first prefix-free binary codes had this property: Fano coding \cite{fano1949transmission}, the direct predecessor of Huffman coding.
Fano coding proceeds by successively sorting the items by frequency, and then partitioning them into two sets with as close to equal frequencies as possible.
While Huffman coding yields an optimal prefix-free binary code (when symbols are coded individually), Fano coding yields a tree with near optimal expected depth, at most 1 worse than optimal \cite{fano_proof}.
However, due to its divisively constructed nature, it is much easier to merge with existing similarity-based tree construction techniques.
In this work we take a step towards understanding this trade-off between statistical error and computational costs, and 
provide a scheme for constructing a PLT that interpolates between the two extremes.

\section{Two extremal trees}\label{sec:clusteringSchemes}

To better describe our interpolating algorithm, we begin by describing in more detail its two endpoints: the similarity-based construction of recursive balanced 2-means clustering, and the coding-theoretic trees from Fano coding.
Defining mathematical notation, in this problem we are given $N$ data points $\b{x_1},\hdots,\b{x_N}\in\R^d$, with corresponding label vectors $\b{y_i} \in \{0,1\}^L$ with $y_{i,\ell}$ indicating whether label $\ell$ is relevant to $\b{x_i}$.
We use boldface to denote vectors, and define $\mathbf{1}$ as the all ones vector of appropriate dimension.
In our binary clustering algorithms, we use the cluster assignment vector $\b{\alpha}\in\{-1,+1\}^L$ to denote the assignment of label $\ell$ to cluster $\alpha_\ell$, where the left child contains all labels $\{\ell : \alpha_\ell = +1\}$ and the right child contains all labels $\{\ell : \alpha_\ell = -1\}$.
The cluster centers we optimize over are $\b{\mu_{+}},\b{\mu_{-}}\in\R^d$.
Utilizing the label matrix $Y\in \{0,1\}^{N \times L}$, we construct the vector of marginal label frequencies $\b{f} \propto Y^\top \mathbf{1}$.
We now describe the two extremal schemes.

\subsection{Balanced spherical 2-means clustering}
Many existing similarity-based XMC tree models use some variant of balanced 2-means clustering to construct trees for their models.
These methods cluster the labels utilizing high dimensional label embeddings, which are constructed such that similar labels have similar label embeddings.
In this work, we utilize Positive Instance Feature Aggregation (PIFA) embeddings $\b{v_\ell}\in\R^d$ for label $\ell$, which are constructed for a given label by averaging the training points that are relevant for that given label (further discussion in \cite{parabel}).

Balanced spherical 2-means clustering is a common variant of $k$-means clustering for $k=2$ using cosine similarity.
In this setting the objects to be clustered, our $L$ label embeddings $\{\b{v_\ell}\}_{\ell=1}^L$, are all rescaled to have unit $\ell_2$ norm, and our cluster centers $\b{\mu_\pm}$ are restricted similarly.
The objective is to find a cluster assignment vector $\b{\alpha}$ and corresponding $\b{\mu_\pm}$ that maximize the sum of similarities between $\b{v_\ell}$ and it's corresponding cluster center ($\b{\mu_{+}}$ if $\alpha_\ell=+1$, $\b{\mu_{-}}$ if $\alpha_\ell=-1$).
An additional balance constraint is enforced by requiring that $|\b{\alpha}^\top \mathbf{1}|\le 1$, restricting the two clusters to be of equal size.
Mathematically, this optimization problem of balanced spherical 2-means clustering can be formulated as below

\vspace{-.1cm}
\begin{equation}
\max_{\substack{\| \b{\mu_\pm} \|_2 = 1\\ \b{\alpha}\in \{ -1,+1\}^L \\ |\b{\alpha}^\top \mathbf{1}| \le 1}} \sum_{\ell=1}^L \frac{1}{L}\left( \frac{1+\alpha_\ell}{2} \b{v_\ell}^\top \b{\mu_{+}} + \frac{1 - \alpha_\ell}{2} \b{v_\ell}^\top \b{\mu_{-}}\right) \label{eq:sp2meanObj}
\end{equation}
where $\alpha_\ell$ corresponds to the cluster assignment of label $\ell$, and defines the partitioning of labels to the left and right children.
Due to the combinatorial constraint of $\b{\alpha}\in \{-1,+1\}^L$, combined with the balanced constraint of $|\b{\alpha}^\top \mathbf{1}| \le 1$, this is an NP-hard problem \cite{bertoni2012size}.
Fortunately, we can still generate a good approximate solution efficiently via alternating maximization, leading to convergence to a local maxima.
For a fixed cluster assignment $\b \alpha$, we have that $\b{\mu_\pm}$ are optimized as being proportional to their respective cluster centers.
Concretely, $\b{\mu_{+}} \propto \sum_{\ell:\alpha_i=+1} \b{v_\ell}$
with $\b{\mu_{-}}$ optimized similarly.
Optimizing $\b \alpha$ for fixed $\b{\mu_\pm}$ requires the following:
\begin{enumerate}
    \item Sort the labels according to $\b{v_\ell}^\top (\b{\mu_{+}} - \b{\mu_{-}})$
    \item Assign the first half of the labels in this sorted order as $\alpha_\ell=+1$, and the latter half as $\alpha_\ell=-1$
    \item If $L$ is odd, assign the middle label ($\ell = \lfloor L/2\rfloor+1$) as $\alpha_\ell=-1+2\cdot\mathds{1}\{\b{v_\ell}^\top (\b{\mu_{+}} - \b{\mu_{-}}) > 0\}$
\end{enumerate}
By performing alternating maximization until the objective value increase from iteration to iteration falls below a specified threshold, we are able to efficiently generate a high quality clustering $\b{\alpha}$, and use this to partition the labels.
Constructing a tree by recursively applying balanced spherical 2-means clustering will ensure that similar labels are close to each other in the tree, but will completely ignore the label frequencies and place all labels at the same depth.

\subsection{Fano tree}

We now examine the details of a Fano coding based tree \cite{fano1949transmission}, which is designed to improve computational efficiency by reducing the expected depth required to traverse at test time.
This scheme proceeds divisively similarly to spherical 2-means clustering, but instead of splitting the labels into the two most similar groups, it splits them into two sets of equal frequency.

This is accomplished in Fano coding by sorting the labels by frequency, and iteratively adding the the largest remaining frequency to the left child until its total frequency surpasses 1/2, sending the remaining labels to the right child.
Reformulating this into an optimization problem outright is difficult due to the combinatorial nature of this task, and so instead we relax the constraints and allow ourselves to fractionally allocate the middle label to give each child a frequency of exactly 1/2.
In order to force higher frequency labels to the left child, having $\alpha_\ell=+1$, we assign value $f_\ell^2$ to this choice, noting that any super-linear function of $f_\ell$ can be used.
This leads to our reformulation of each recursive call of the Fano coding scheme as a linear program (LP): 
\vspace{-.1cm}
\begin{equation}\label{eq:fanoOj}
\max_{ \substack{\b{\alpha}\in [-1,+1]^L \\ \b{\alpha}^\top \b{f}=0}}  \sum_{\ell=1}^L \alpha_\ell f_\ell^2.
\end{equation}
By the fundamental theorem of linear programming the objective attains its maxima on a corner point, and so \eqref{eq:fanoOj} is maximized by letting $\alpha_\ell=+1$ for the largest frequency items and $-1$ for the smaller ones.
Solving this LP may yield one fractional $\alpha_\ell$ (assuming the $f_\ell$ are unique) due to the $\b{\alpha}^\top \b{f}=0$ constraint, which when changed to $+1$ will yield a valid Fano code assignment. 
Constructing a tree by recursively applying this Fano coding scheme yields a tree that prioritizes placing frequent labels at shallow depths to optimize computational efficiency, but ignores the label embeddings.

One important practical consideration is how we should construct the frequency vector our algorithm utilizes to minimize expected depth. While at first glance one may simply want to use the marginal label frequencies (number of occurrences of the label in the dataset), this can yield poor performance.
This is because we want to minimize the expected depth one needs to search to in order to find $k$ relevant labels for a given context.
Considering the simple case of $k=1$, if we have two very frequent labels that always show up together we do not need to put both of them high up in the tree.
We do not care about the label's marginal frequency; we only care about if it will prevent contexts from needing to search deeper to find $k$ relevant labels.
There are several possible ways to create such a frequency vector suited for minimizing expected depth, which we denote $\b{\widetilde{f}}$.
To preserve the flow of this work, we relegate the construction of $\b{\widetilde{f}}$ and further discussion on it to Appendix \ref{app:ftilde}.

\section{Interpolated label trees} \label{sec:interpolatedPLT}

Balanced 2-means clustering and Fano coding are two seemingly disparate algorithms with the former operating solely on the label embeddings and the latter operating solely on the label frequencies.
We would ideally have a scheme that utilizes both similarity and frequency information, with a tunable knob offering a range of operating points trading off between precision and efficiency so that the right model can be selected for the specific application at hand.
In online inference for example, a high-efficiency solution may be desired, even at the expense of a slight drop in accuracy, while for batch inference tasks a high-accuracy solution may be more useful, due to the less stringent latency constraints.
A priori, it is unclear how to construct any intermediary point, let alone to interpolate between the two.
In this section, we show that we can achieve many desirable intermediary operating points, as shown visually in Figure \ref{fig:newOperatingPoints}.

\subsection{Weighted 2-means clustering}
We begin by constructing one such intermediary point, defining a new objective we call weighted 2-means clustering. 
This can best be understood by considering a simple case where each frequency is some integer multiple of a common base frequency. 
Then, one potential clustering scheme is to duplicate each point proportional to its frequency, and run balanced 2-means clustering on this expanded set of points, where after each iteration we assign each original point to the cluster where the majority of its duplicates fall.
Mathematically, our weighted 2-means objective can be formulated as an LP: 
\vspace{-.1cm}
\begin{equation}\label{eq:weighted2MeansObj}
\max_{ \substack{\| \b{\mu_\pm} \|_2 = 1 \\ \b{\alpha}\in [-1,+1]^L\\ \b{\alpha}^\top \b{f}=0}}  \sum_{\ell=1}^L f_\ell\left( \frac{1+\alpha_\ell}{2} \b{v_\ell}^\top \b{\mu_{+}} + \frac{1 - \alpha_\ell}{2} \b{v_\ell}^\top \b{\mu_{-}}\right).
\end{equation}
This constitutes making the following two changes to the balanced 2-means LP in \eqref{eq:sp2meanObj}.
\begin{enumerate}
\item Follow a frequency weighted 2-means objective; weight how well label $\ell$ is matched by its frequency $f_\ell$ instead of $\frac{1}{L}$.
\item Change the constraint $|\b{\alpha}^T \b{1}| \le 1$ to $\b{\alpha}^T \b{f} =0$, and relax $\alpha_\ell$ to be contained in $[-1,1]$.
\end{enumerate}

Examining our alternating minimization algorithm, we have that for a fixed cluster assignment $\b \alpha$, $\b{\mu_\pm}$ are now optimized as the \textit{weighted} means of their clusters.
For fixed $\b{\mu_\pm}$, $\b \alpha$ is optimized similarly to before. Whereas for balanced 2-means we sorted the indices by $(\b{\mu_{+}} - \b{\mu_{-}})^\top \b{v_\ell}$ and assigned the smaller half to cluster 1, we now assign as many labels as we can, proceeding in decreasing order of $(\b{\mu_{+}} - \b{\mu_{-}})^\top \b{v_\ell}$, until their frequencies sum to over 1/2, and fractionally divide the $\alpha_\ell$ of this last label $\ell$ to achieve $\b{\alpha}^\top \b{f}=0$.

We relax the combinatorial problem of optimizing over $\alpha_\ell\in\{-1, +1\}$ to the LP optimizing over $\alpha_\ell \in [-1,+1]$, as otherwise we cannot ensure this frequency balance constraint.
If we relax the balance constraint to be $|\b{f}^\top \b{\alpha}|<c$ and maintain the constraint on our $\alpha_\ell$ to be discrete, then even alternating maximization will be difficult.
This is because for fixed $\b{\mu_\pm}$, due to the unequal costs (frequencies) of labels, assigning the $\alpha_\ell$'s becomes a knapsack problem.

This formulation neatly obeys our intuition on what a similarity and frequency based clustering should look like.
It places higher frequency labels higher in the tree, while also placing similar labels near each other.
Hence, this weighted 2-means clustering falls in between the heavily imbalanced Fano tree and the the fully balanced tree from balanced 2-means.

\subsection{Interpolating with a combined objective}
Equipped with this one intermediary point, we now develop a more fine grained control of the trade-off between frequency and similarity information.
We formulate below a combined objective that interpolates between these three operating points using a user specified hyperparameter $\lambda \in [0,2]$. 
We use standard mathematical notation with $(x)_{+}:=\max(x,0)$ and $x \wedge y := \min(x,y)$.
\begin{align*}
\max_{ \substack{\|\b{\mu_\pm}\|_2=1 \\ \b{\alpha}\in [-1,+1]^L \\ \b{\alpha}^\top \b{f}(\lambda)=0}}  &\sum_{\ell=1}^L f_\ell(\lambda)(2-\lambda)\left( \frac{1+\alpha_\ell}{2} \b{v_\ell}^\top \b{\mu_{+}} + \frac{1 - \alpha_\ell}{2} \b{v_\ell}^\top \b{\mu_{-}}\right)   \\[-.4cm]
&\hspace{2cm} \vspace{-1cm}+ \alpha_\ell(\lambda-1)_{+}f_\ell(\lambda)^2 \numberthis \label{eq:combinedObjective} \\[.1cm]
\text{where }&f_\ell(\lambda) = \frac{(2-\lambda)f_\ell^{\lambda \wedge 1} + (\lambda-1)_{+} \widetilde{f_\ell}+\gamma/L}{(2-\lambda)\sum_{j=1}^L f_j^{\lambda \wedge 1} + (\lambda-1)_{+}+\gamma}
\end{align*}
To understand the above LP, first consider setting the additive smoothing parameter $\gamma=0$.
Then, this LP yields balanced 2-means clustering for $\lambda=0$, our intermediary weighted 2-means clustering at $\lambda=1$, and Fano coding at $\lambda=2$.
It smoothly interpolates between these three points, and for $\lambda \in (0,1)$ can be seen as applying a shrinkage function on the frequencies with $f_\ell(\lambda) \propto f_\ell^\lambda+\gamma/L$.
The interpolation is achieved by linearly placing less weight on the 2-means objective as $\lambda$ ranges from $0$ to $2$, and gradually changing the frequency vector $\b{f}(\lambda)$ from constant at $\lambda=0$ to $\b{f}(\lambda)\propto \b{f} + \gamma/L$ when $\lambda=1$ (balanced 2-means to weighted 2-means).
$\widetilde{\b f}$ plays no role when $\lambda \in [0,1]$, as the nonlinear $(\lambda-1)_{+}$ is only active when $\lambda \in (1,2]$. 
In this regime, the frequency vector linearly trades off between $\b f$ and $\widetilde{\b f}$, as our objective is $(2-\lambda)$ times the weighted 2-means objective plus $(\lambda-1)_{+}$ times the Fano objective on $f_\ell(\lambda)$.
The additive Laplacian smoothing parameter $\gamma>0$ is to guard against distributional mismatch between train and test sets.

Constraining the coordinates of $\b{\alpha}$ to be in $\{-1,+1\}$ leads to a combinatorial optimization problem, with the inherent difficulties this brings (as mentioned before).
Fortunately however, relaxing the constraints to $\b{\alpha} \in [-1,+1]^L$ until rounding the final solution makes this objective easily optimizable via alternating maximization, stopping when the objective value no longer increases.
Even though the optimization problem in \eqref{eq:combinedObjective} is nonconcave, the objective is bilinear in $\b{\alpha}$ and $\b{\mu_\pm}$, and so each phase of alternating maximization is optimizing a linear objective (possibly with a quadratic constraint), which we show can be done efficiently.

\noindent\textbf{Optimizing $\b{\mu_\pm}$: } fixing $\b{\alpha}$, and decoupling $\b{\mu_{+}}$ and $\b{\mu_{-}}$, we see that $\b{\mu_{+}}$ is optimized as
\begin{align*}
\argmax_{\|\b{\mu_{+}}\|_2=1} &\sum_{\ell=1}^L f_\ell(\lambda)(1+\alpha_\ell)\b{v_\ell}^\top \b{\mu_{+}} \propto \sum_{\ell=1}^L f_\ell(\lambda)(1+\alpha_\ell)\b{v_\ell},
\end{align*}
where $\b{\mu_{-}}$ can be optimized similarly. 
This means that once given cluster assignments $\b{\alpha}$, $\b{\mu_\pm}$ are optimized as being proportional to their weighted cluster centers, requiring $O(Ld)$ time.

\noindent\textbf{Optimizing $\b{\alpha}$: } fixing $\b{\mu_\pm}$, the optimal $\b{\alpha}$  can be solved for as
\begin{align*}
\alpha\opt\hspace{-.1cm}
=\hspace{-.1cm}\argmax_{ \substack{\b{\alpha}\in [-1,+1]^L \\ \b{\alpha}^\top \b{f}(\lambda)=0}}&
\sum_{\ell=1}^L \alpha_\ell \underbrace{\left(  f_\ell(\lambda)\frac{2-\lambda}{2}\b{v_\ell}^\top (\b{\mu_{+}}-\b{\mu_{-}}) + (\lambda-1)_{+}f_\ell(\lambda)^2\right)}_{\beta_\ell}\\
=\argmax_{ \substack{\b{\alpha}\in [-1,+1]^L \\ \b{\alpha}^\top \b{f}(\lambda)=0}}&
\b{\alpha}^\top \b{\beta} \numberthis \label{eq:mainTextalphaLP}
\end{align*}
Due to the nice structure of this LP, once we compute $\b{\beta}$ which requires $O(Ld)$ time, we are able to solve the LP efficiently in $O(L)$ time, as stated in the following lemma.

\begin{lemma}\label{lem:alphaLP}
	An optimal $\b{\alpha}\opt$ for \eqref{eq:mainTextalphaLP} can be constructed as:
	\begin{enumerate}
		\item Sort labels by $\beta_\ell / f_\ell(\lambda)$, initialize all $\alpha\opt_\ell=-1$
		\item Starting from the largest $\beta_\ell / f_\ell(\lambda)$, iteratively assign $\alpha\opt_\ell=+1$ until $\b{\alpha\opt}^\top \b{f}(\lambda)>0$
		\item Assign the last label $\ell$ that was set to $\alpha\opt_\ell=1$ fractionally to achieve $\b{\alpha\opt}^\top \b{f}(\lambda)=0$
	\end{enumerate}
\end{lemma}
\begin{proof}[Proof of Lemma \ref{lem:alphaLP}]
We show that this $\b{\alpha\opt}$ is an optimal solution for \eqref{eq:mainTextalphaLP} by analyzing the dual LP.
To start, we modify \eqref{eq:mainTextalphaLP} to obtain an LP in standard form by shifting the $\alpha_\ell$ to range between $[0,1]$ instead of $[-1,1]$ (the original $\b{\alpha}$ can be obtained by shifting and rescaling). We replace $\b{f}(\lambda)$ by $\b{f}$ for brevity, and denote the optimal primal value as $p\opt$, where
\begin{equation}
p\opt= \max_{ \substack{0\le \b{\alpha}\le 1 \\ \b{\alpha}^\top \b{f}=1/2}} \b{\alpha}^\top \b{\beta}.\numberthis \label{eq:LPtoDual}
\end{equation}
We prove Lemma \ref{lem:alphaLP} by showing that there exists a feasible solution to the dual LP of \eqref{eq:LPtoDual}, which achieves objective value equal to that of our $\b{\alpha}\opt$ in the primal.
We can construct such an optimal $\b{\alpha}\opt$ by initializing all coordinates to 0, sorting the indices by $\beta_\ell/f_\ell$, iteratively assigning $\alpha\opt_\ell=1$ until $\b{\alpha\opt}^\top \b{f}>1/2$, then setting this final adjusted index (which we call $\lt$) to have $\alpha\opt_\lt=\frac{1}{f_\lt}\left(\frac{1}{2}-\sum_{i \in \CS} f_i \right)$.
Defining the set of indices where $\alpha\opt_i=1$ from the above scheme as $\CS$, we obtain $p\opt\ge\sum_{i \in \CS} \beta_i + \alpha_\lt\opt \beta_\lt$. Denoting the $L\times L$ identity matrix by $I_L$ and defining 
$A=\begin{bmatrix}
I_L &f & -f
\end{bmatrix}^\top, \b{b} =  \begin{bmatrix}
\mathbf{1}_L^\top & \frac{1}{2} & -\frac{1}{2}
\end{bmatrix}^\top$ 
we obtain our dual LP with value $d\opt$
\begin{equation}\label{eq:dual}
d\opt = \min_{\substack{ \b{y} \ge 0 \\
		A^\top \b{y}\ge \b{\beta}}} \b{y}^\top \b{b}.
\end{equation}
Denoting by  $\b{y}\opt$ our proposed optimal solution to eq. \eqref{eq:dual}, we see by complementary slackness that we can set $y_i\opt=0$ whenever $[A\b{\alpha\opt}-\b{\beta}]_i>0$.
Considering the case where there is a fractional $\alpha_i\opt$, that is $\sum_{i \in \CS} f_i<\frac{1}{2}$ and so $\alpha_\lt\in(0,1)$, we must have that $y_i\opt=0$ for $i\in \CS^\mathsf{c}$ and for $i=\lt$.
To construct the remaining entries of $\b{y}$, we set $y_{L+1}\opt= \beta_\lt / f_\lt$ with $y_{L+2}=0$, and $y_i\opt=\beta_i-f_i y_{L+1}\opt$ for $i \in \CS$; this allows us to satisfy $A^\top \b{y}\opt \ge \b{\beta}$.
We then see that our dual objective value is
\begin{equation}
    d\opt
\le \b{b}^\top \b{y}\opt
=\left(\sum_{i \in \CS} \beta_i-f_i y_{L+1}\opt \right) + \frac{1}{2}y_{L+1}\opt
=\sum_{i \in \CS} \beta_i + \alpha_\lt\opt \beta_\lt. \label{eq:dualUB}
\end{equation}

Since the optimal dual objective value upper bounds the optimal primal objective value (i.e., $d\opt \ge p\opt$) and our $\b{\alpha\opt}$ achieves the upper bound given by the dual in \eqref{eq:dualUB}, $\b{\alpha\opt}$ is an optimal solution for the primal \cite{boyd2004convex}.
Where this proof assumed that $\alpha_\lt\in(0,1)$, the case of $\alpha_\lt =0$ follows identically, with $\alpha_\lt=1$ requiring us to increment the index $\lt$ to the next element in sorted order of $\beta_\ell/ f_\ell$.	
\end{proof}

Now that we have shown that each step of alternating maximization can be performed efficiently, we prove convergence of the overall procedure in the following theorem.
\begin{theorem}
	Performing alternating maximization on $\b{\alpha}, \b{\mu_\pm}$ for the optimization problem in \eqref{eq:combinedObjective} converges to a stationary point.
\end{theorem}
\begin{proof}
	Defining one iteration as comprising both a $\b{\mu_\pm}$ and an $\b{\alpha}$ optimization phase, we see similarly to the proof of Theorem 2.2 in \cite{parabel} that since the objective value must increase over each iteration, no configuration of $\b{\alpha}$ will be repeated.
	This is because the $\b{\mu_\pm}$ at the end of two iterations must be the same if the $\b{\alpha}$ at the end of those two iterations are the same. However, this would mean that the objective value did not increase in this iteration, and so the procedure would have terminated.
	We observe that there are at most $2^{L-1}$ possible $\b{\alpha}$ vectors using our iteration scheme, as the $\b{\alpha}$ vectors we construct have at most 1 fractional entry, which (if feasible) is uniquely determined by the other $L-1$ entries due to the balancedness constraint of $\b{\alpha}^\top \b{f}(\lambda)=0$.
	As every non-fractional entry will be $\pm 1$, this gives $2^{L-1}$ total $\b{\alpha}$ that can appear at the end of an iteration.
	Thus, the algorithm will converge to a stationary point in a finite number of iterations.
\end{proof}
This result can be thought of as convergence to a Nash Equilibrium in a nonconcave game, where two players, one controlling $\b{\mu_\pm}$ and one controlling $\b{\alpha}$, alternate best responding to the other's actions.
Due to the nonconcave objective and constraint set our guarantees for convergence of alternating maximization are to a stationary point and not to a local (or global) maxima.

\section{Experiments}\label{sec:simulations}

We demonstrate the effectiveness of our scheme by using it to augment two widely used tree-based XMC models.
Where the similarity-based trees optimize for precision and coding theoretic trees optimize for expected depth, we show that by varying $\lambda$ we can easily interpolate between these two extremes.
We provide results on the public benchmark XMC datasets AmazonCat-13K \cite{mcauley2013hidden_amazoncat13k}, Amazon-670K \cite{mcauley2013hidden_amazoncat13k}, and Wiki-500K \cite{Bhatia16_XMCRepo}.
We also provide results on 5 large e-commerce datasets with up to 10 million labels.
Table \ref{tab:datasets} lists the number of training points $N$, the dimension of the points $d$, the number of labels $L$, the number of test points $N'$, and the average number of labels per training point for each public dataset. Table \ref{tab:amazon_datasets} lists the same statistics for each e-commerce dataset.

\begin{table}[h]
\begin{tabular}{l|lllll}
\hline
Dataset       & $N$         & $d$         & $L$       & $N'$      & Avg labels/pt  \\ \hline
AmazonCat-13K & 1.2M & 204K   & 13K   & 307K & 5.04                                        \\
Wiki-500K     & 1.8M & 2.4M & 501K & 784K & 4.77                                           \\
Amazon-670K   & 490K   & 136K   & 670K & 153K & 5.45 \\ \hline
\end{tabular}
\caption{Public XMC dataset statistics. }
\label{tab:datasets}
\vspace{-.8cm}
\end{table}

We first use our algorithm to augment the similarity-based XMC method Parabel \cite{parabel}. In particular, we replace Parabel's tree with trees constructed by our algorithm, improving their expected depth on public datasets.
We then use our algorithm to augment the coding-theoretic tree used by fastText. In particular, we replace fastText's Huffman-based hierarchical softmax with our algorithm's PLT, improving the statistical performance of the model on public datasets.
{Finally, we demonstrate the effectiveness of the augmented Parabel model on large e-commerce datasets, where we show an improvement in expected depth of up to 20\%.}

To measure statistical performance we use precision at $k$, which is computed as the fraction of true positive labels out of the $k$ labels predicted by a model, averaged over the test contexts.
To quantify computational efficiency, we utilize expected depth at $k$ as a proxy for expected prediction latency, defined as the average depth (over the test contexts) searched to in our PLT.
For a given context, this is obtained by looking at the top $k$ predicted labels, and taking the depth of the deepest returned label.
Creating a PLT prediction algorithm that realizes these expected depth gains as wall-clock improvements is left as future work, as this is an application and implementation-specific task.
Additional experiment details can be found in Section \ref{sec:experimental_details}.

\vspace{-.1cm}
\subsection{Augmented Parabel}\label{sec:simulations_parabel}

We augment Parabel with our trees for different values of $\lambda$ and evaluate the performance on the large-scale public XMC datasets Amazon-670K and Wiki-500K in Figure \ref{fig:parabel}.
We interpolate the full spectrum from fully similarity-based ($\lambda=0$, standard Parabel) to coding-theoretic ($\lambda=2$).

While there is a general Pareto-style trade-off between expected depth and precision, we observe a surprising phenomenon in our numerical experiments; we are able to marginally \textit{improve} precision while reducing expected depth for small $\lambda>0$.
This means that, even without translating these expected depth gains into wall-clock improvements, we are able to improve model precision at effectively no cost.
One possible explanation for this improvement is the fact that our interpolated scheme, compared to Parabel's clustering algorithm, reduces the depth of frequently accessed ``popular'' labels.
Since each level in the tree compounds the error in routing and prediction, having frequent labels higher in the tree improves their prediction accuracy.

Additionally, $\widetilde{\b f}$ can be constructed via several different methods, as previously discussed.
In Figure \ref{fig:parabel} we show several minor modifications to our original scheme (depicted in blue).
In the orange curve, we construct $\widetilde{\b{f}}$ in a greedy manner; that is, we iteratively find the highest frequency label in the dataset, remove it and all contexts that contain it from the dataset (assigning it frequency equal to the number of removed contexts), and repeat.
In the blue curve, we instead construct $\widetilde{\b{f}}$ by sorting the labels by their marginal frequencies, then iterate over each context and assign its frequency to the label it contains with the highest marginal frequency.

\begin{figure}[b]
\captionsetup[subfigure]{aboveskip=1pt}
	\centering
	\vspace{-.3cm}
	\begin{subfigure}[b]{0.23\textwidth}
		\includegraphics[width=\columnwidth, trim=.7cm .3cm 1.2cm 1.5cm, clip]{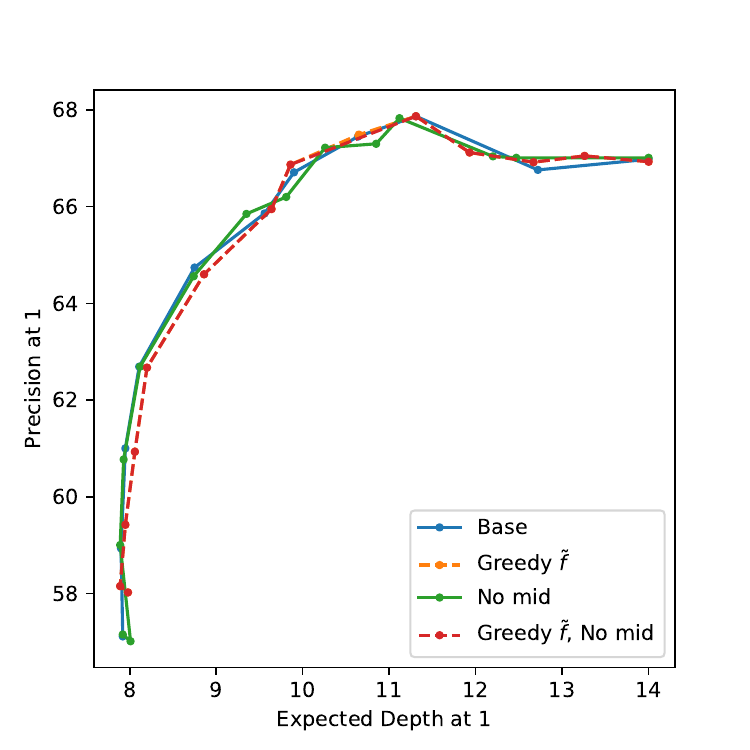}
		\caption{p@1 for Wiki-500K}
		\vspace{.05cm}
	\end{subfigure}
	\hfill
	\begin{subfigure}[b]{0.23\textwidth}
		\includegraphics[width=\columnwidth, trim=.7cm .3cm 1.2cm 1.5cm, clip]{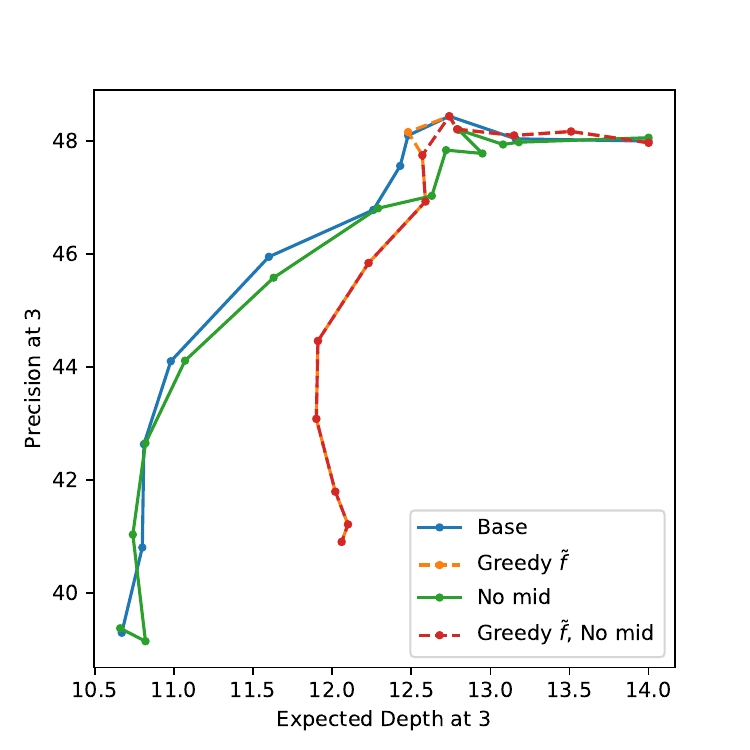}
		\caption{p@3 for Wiki-500K}
		\vspace{.05cm}
	\end{subfigure}
	\begin{subfigure}[b]{0.23\textwidth} 
		\includegraphics[width=\columnwidth, trim=.4cm .3cm 1.2cm 1.5cm, clip]{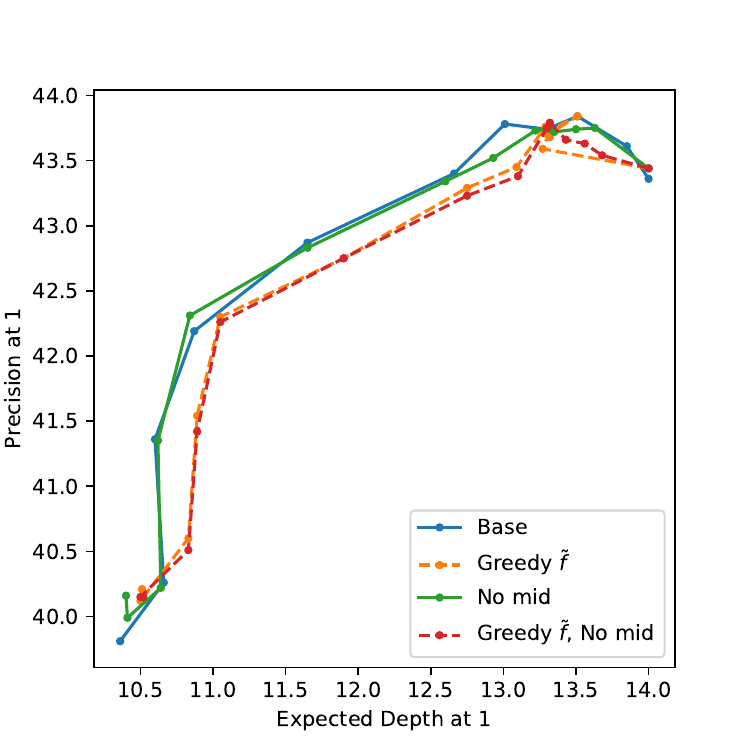}
		\caption{p@1 for Amazon-670K}
	\end{subfigure}
	\hfill
	\begin{subfigure}[b]{0.23\textwidth}
		\includegraphics[width=\columnwidth, trim=.4cm .3cm 1.2cm 1.5cm, clip]{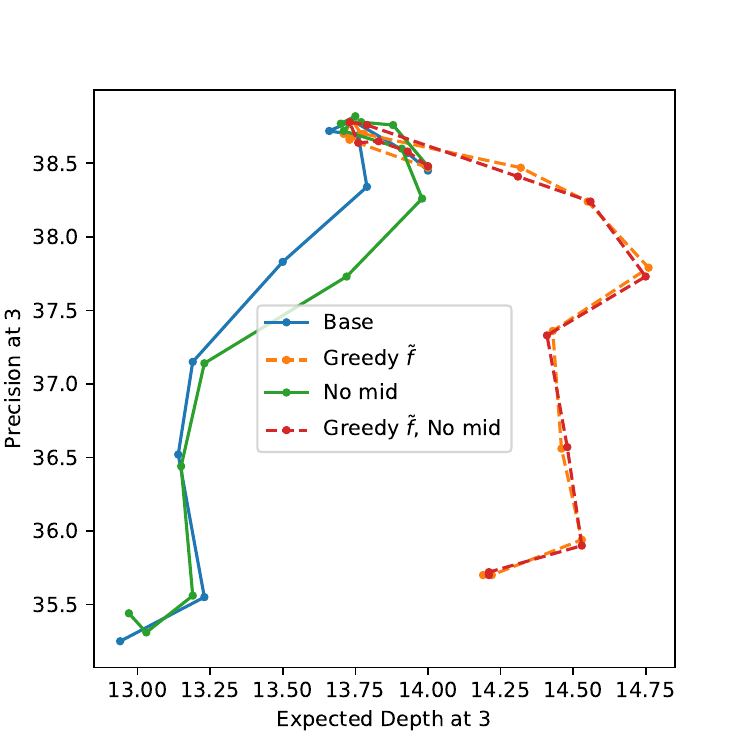}
		\caption{p@3 for Amazon-670K}
	\end{subfigure}
	\vspace{-.35cm}
	\caption{Model expected depth versus precision for Parabel augmented with our algorithm.}
	\label{fig:parabel}
	\vspace{-.4cm}
\end{figure}

Another question one can ask is whether this intermediary weighted 2-means clustering point is necessary.
That is, what if we directly interpolate between our Fano coding tree and balanced 2-means clustering?
The results for this are shown in the green curve, which utilizes the marginal frequency based $\widetilde{\b{f}}$, and does not interpolate through our intermediary weighted 2-means clustering point.
Where the orange curve is a modification to the blue curve, the red curve is a similar modification of the green curve. 
For this red curve, we interpolate directly from a Fano coding tree from a greedily constructed $\widetilde{\b{f}}$ to a balanced 2-means clustering tree. Note that the red curve almost fully overlays the orange one.


We see that while these 4 schemes perform very similarly in terms of their precision at 1, their precision at 3 differs dramatically.
The schemes utilizing the greedily constructed $\widetilde{\b{f}}$ perform dramatically worse than the schemes using the marginal frequency based $\widetilde{\b{f}}$.
Additionally, we can see that not utilizing the intermediary weighted 2-means objective yields worse performance (going from the blue to the orange curve).

\subsection{Augmented fastText} \label{sec:simulations_fasttext}
We augment fastText with our trees for different values of $\lambda$ and evaluate the performance on the large-scale public XMC datasets AmazonCat-13K and Wiki-500K.
We interpolate the full spectrum from fully similarity-based ($\lambda=0$) to coding-theoretic ($\lambda=2$).
Note that the case of $\lambda=2$ is equivalent to top-down Fano coding but is not equivalent to fastText's bottom-up Huffman coding. Thus, we plot fastText's default Huffman algorithm for comparison.
Figures \ref{fig:amazoncat13k_fasttext} and \ref{fig:wiki500k_fasttext} show the results on AmazonCat-13K and Wiki-500K, respectively.
We observe similar empirical results as in the setting of augmenting Parabel with our algorithm, noting that the expected depths are higher in this setting due to the leaf size of 1 for fastText in contrast to the leaf size of 100 for Parabel.

\begin{figure}[h]
	\centering
	\vspace{-.2cm} 
	\begin{subfigure}[b]{0.23\textwidth}
		\includegraphics[width=\columnwidth, trim=.4cm .3cm 1.2cm 1.4cm, clip]{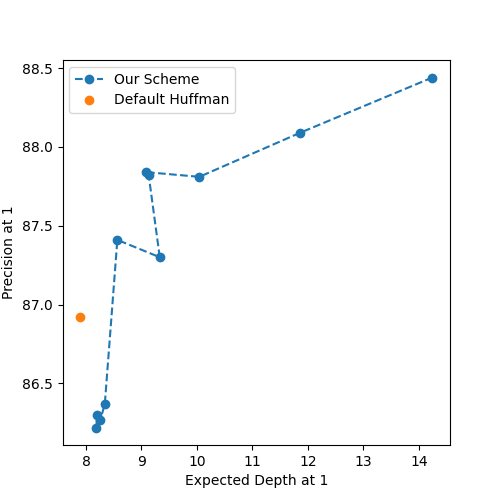}
		\vspace{-.55cm}
		\caption{p@1 for AmazonCat-13K}
		\label{fig:amazoncat13k_fasttext}
	\end{subfigure}
	\hfill
	\begin{subfigure}[b]{0.23\textwidth}
		\includegraphics[width=\columnwidth, trim=.4cm .3cm 1.2cm 1.4cm, clip]{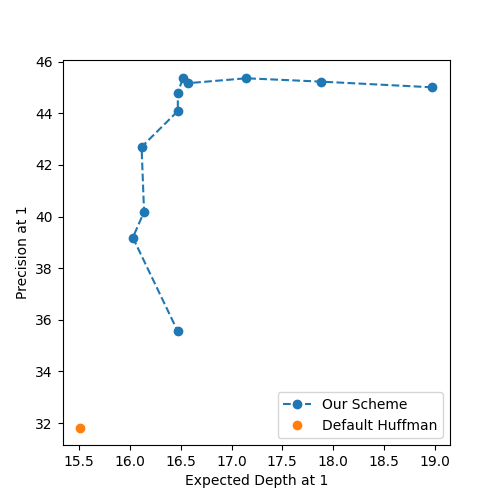}
		\vspace{-.55cm}
		\caption{p@1 for Wiki-500K}
		\label{fig:wiki500k_fasttext}
	\end{subfigure}
\vspace{-.4cm}
\caption{Model expected depth versus precision@1 for fastText augmented with our algorithm.}	
\vspace{-.6cm}
\end{figure}


\subsection{Applying augmented Parabel to e-commerce customer logs}\label{sec:simulations_amazon}
\newlength{\figH}
\setlength{\figH}{2.6cm}
\newlength{\subFigW}
\setlength{\subFigW}{0.2\textwidth}
\begin{figure}[b!]
\captionsetup[subfigure]{aboveskip=1pt,belowskip=0pt}
	\centering
	\vspace{-.55cm}
	\begin{subfigure}[b]{\subFigW}
		\includegraphics[height=\figH, trim=1.3cm .25cm 1.58cm 1.4cm, clip]{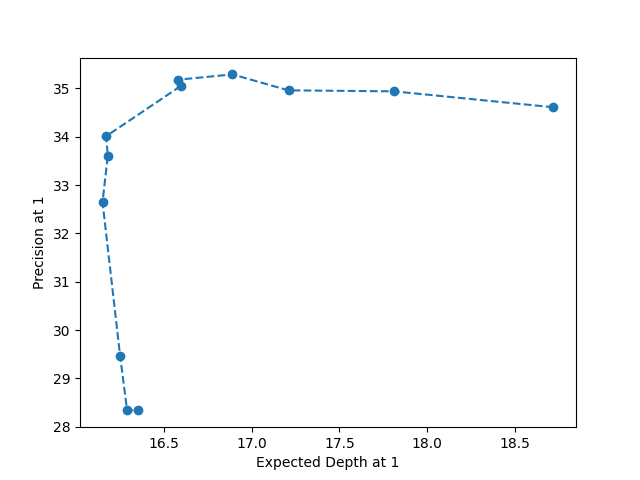}
		\caption{p@1 for Amazon-705K}
	\end{subfigure}
	\hspace{.5cm}
	\begin{subfigure}[b]{\subFigW}
		\includegraphics[height=\figH, trim=1cm .25cm 1.58cm 1.4cm, clip]{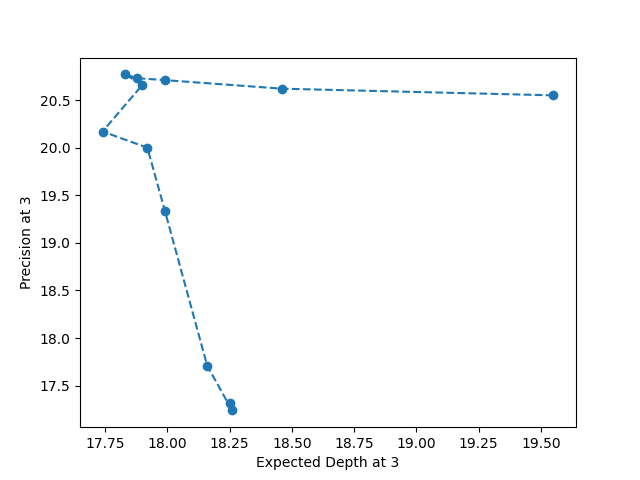}
		\caption{p@3 for Amazon-705K}
	\end{subfigure}
	
	\begin{subfigure}[b]{\subFigW}
		\includegraphics[height=\figH, trim=1.3cm .25cm 1.58cm 1.4cm, clip]{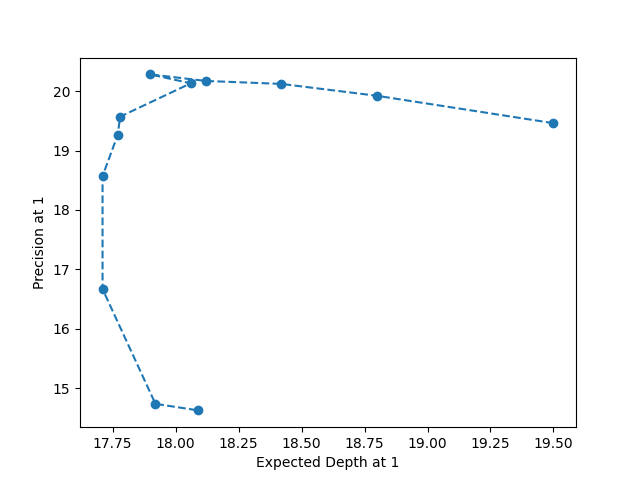}
		\caption{p@1 for Amazon-1M}
	\end{subfigure}
	\hspace{.5cm}
	\begin{subfigure}[b]{\subFigW}
		\includegraphics[height=\figH, trim=1cm .25cm 1.58cm 1.4cm, clip]{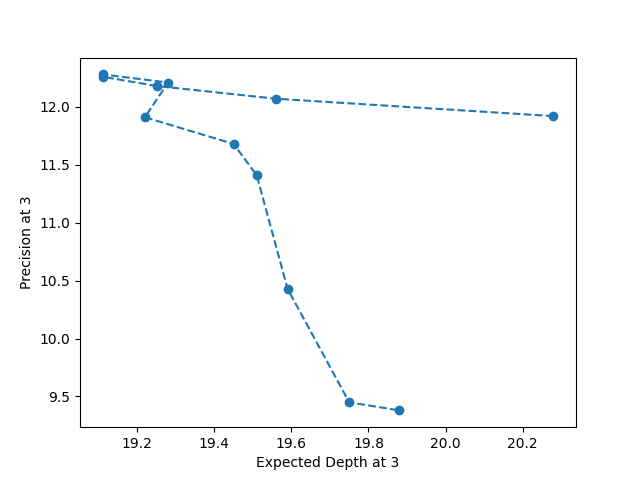}
		\caption{p@3  for Amazon-1M}
	\end{subfigure}
	
	\begin{subfigure}[b]{\subFigW}
		\includegraphics[height=\figH, trim=1.3cm .25cm 1.58cm 1.4cm, clip]{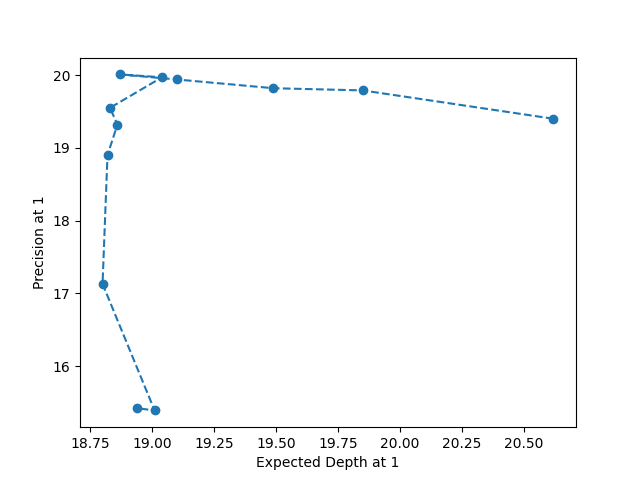}
		\caption{p@1  for Amazon-2M}
	\end{subfigure}
	\hspace{.5cm}
	\begin{subfigure}[b]{\subFigW}
		\includegraphics[height=\figH, trim=1cm .25cm 1.58cm 1.4cm, clip]{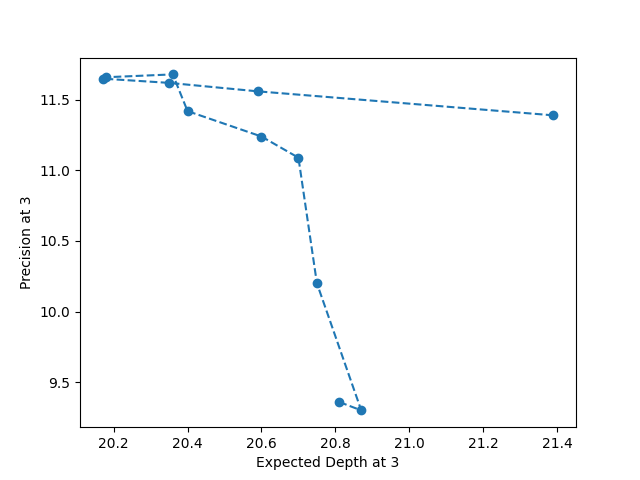}
		\caption{p@3 for Amazon-2M}
	\end{subfigure}
	
	\begin{subfigure}[b]{\subFigW}
		\includegraphics[height=\figH, trim=1.3cm .25cm 1.58cm 1.4cm, clip]{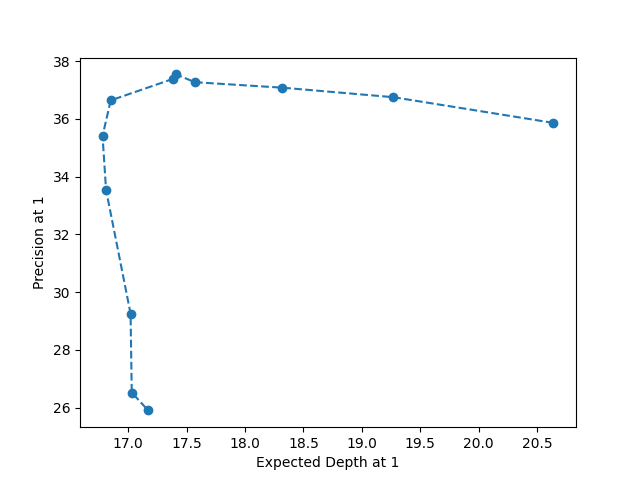}
		\caption{p@1 for Amazon-3M}
	\end{subfigure}
	\hspace{.5cm}
	\begin{subfigure}[b]{\subFigW}
		\includegraphics[height=\figH, trim=1.3cm .25cm 1.58cm 1.4cm, clip]{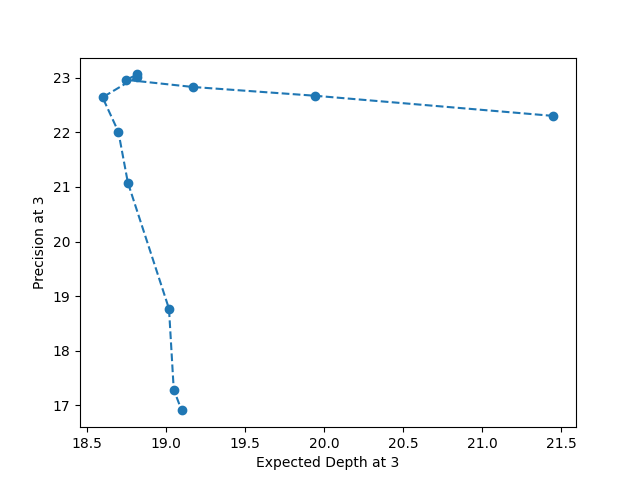}
		\caption{p@3 for Amazon-3M}
	\end{subfigure}
	
	\begin{subfigure}[b]{\subFigW}
		\includegraphics[height=\figH, trim=1cm .25cm 1.58cm 1.4cm, clip]{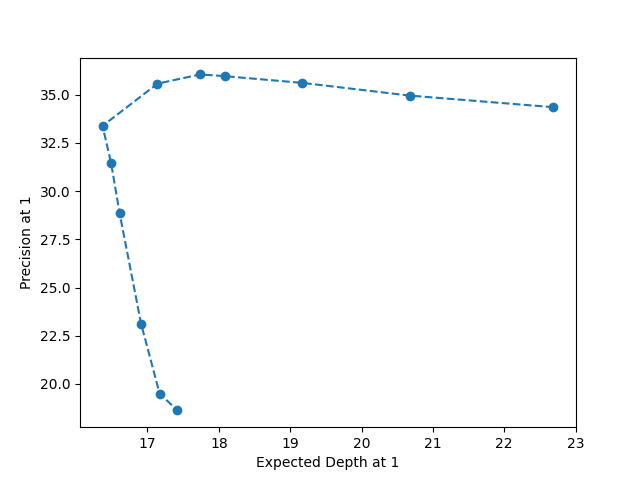}
		\caption{p@1 for Amazon-10M}
	\end{subfigure}
	\hspace{.5cm}
	\begin{subfigure}[b]{\subFigW}
		\includegraphics[height=\figH, trim=1.3cm .25cm 1.58cm 1.4cm, clip]{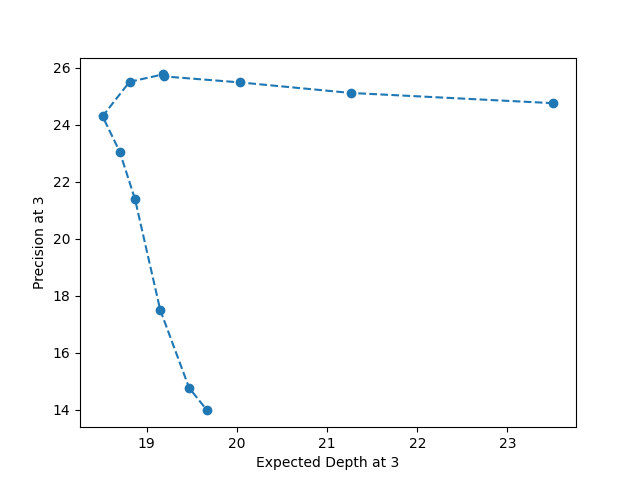}
		\caption{p@3 for Amazon-10M}
	\end{subfigure}
	\vspace{-0.4cm}
	\caption{Expected depth versus precision for Parabel augmented with our algorithm}
	\label{fig:amazon_T}
	\vspace{-.3cm}
\end{figure}

Here we show the effectiveness of our algorithm's augmentation of Parabel in a prototypical e-commerce problem: given a user's context, retrieve a small subset of relevant items from an enormous catalog of products.
XMC models are natural candidates for addressing this problem because the inputs (user contexts) are often high-dimensional and the output space (the catalog of products) is discrete, finite, and extremely large.

\begin{table}[h]
	\begin{tabular}{l|lllll}
		\hline
		Dataset       & $N$         & $d$         & $L$       & $N'$      & Avg labels/pt  \\ \hline
		Amazon-705K & 951K & 207K  & 705K & 50K & 2.94 \\
		Amazon-1M & 1.3M & 200K & 1.2M & 66K & 2.94 \\
		Amazon-2M & 2.4M & 486K & 2.5M & 127K & 2.28 \\
		Amazon-3M & 7.9M &  794K  & 2.7M   & 364K &  3.12  \\
		Amazon-10M & 29.9M  &  1.6M &  9.9M & 1.6M & 6.03  \\ \hline
	\end{tabular}
	\caption{E-commerce dataset statistics. }
	\label{tab:amazon_datasets}
	\vspace{-.8cm}
\end{table}

To generate each dataset, we aggregate one year of Amazon customer product engagement logs, remove context-product pairs that have low engagement, designate a random 5\% of the user contexts as test data, and then use the remaining 95\% as training data.
We represent the user contexts as sparse high-dimensional embeddings.
As seen in Table \ref{tab:amazon_datasets}, we generate datasets ranging from 705 thousand to 10 million products (labels) in order to demonstrate the effectiveness of our algorithm in e-commerce datasets of varying orders of magnitude. Figure \ref{fig:amazon_T} shows the results.

On all these e-commerce datasets, we observe that our algorithm reduces expected depth at 1 by at least 10\% while maintaining precision at 1.
On the largest dataset, our algorithm is able to maintain precision at 1 while reducing expected depth by 20\%.
We similarly observe, to a lesser degree, that our algorithm can reduce expected depth at 3 while marginally improving precision at 3.
Given the large scale of modern e-commerce services, as provided by Amazon and others, these efficiency gains could lead to significant reductions in the infrastructure costs of using tree-based XMC models.

\subsection{Experimental details}\label{sec:experimental_details}

The code used in the numerical results is proprietary and cannot be released to the public. However, we provide details in this section with the aim of making our results reproducible.
\vspace{-.1cm}
\subsubsection{Datasets}
Section \ref{sec:simulations_parabel} and Section \ref{sec:simulations_fasttext} use the train and test datasets provided in \cite{Bhatia16_XMCRepo}.
Section \ref{sec:simulations_parabel} uses the accompanying bag-of-words input features, whereas Section \ref{sec:simulations_fasttext} uses the accompanying raw text features. 

\vspace{-.1cm}
\subsubsection{Models and hyperparameters}

Section \ref{sec:simulations_parabel} and Section \ref{sec:simulations_amazon} use a proprietary implementation of XR-LINEAR \cite{yu2020pecos}, which can be viewed as a generalization of Parabel, and we use the suggested settings to recover the special case of Parabel.
In particular, we use XR-LINEAR with positive instance feature aggregation (PIFA) label representations, teacher forcing negative sampling, and squared hinge loss.
We replace the clustering algorithm with our own, which is equivalent to Parabel's when $\lambda=0$.
For the hyperparameters of Parabel, we set the number of trees to 1 and use the default recommendations for the remaining hyperparameters; i.e., 10 for the maximum number of paths that can be traversed in a tree at prediction time, 100 for the maximum number of labels in a leaf, and 1 for the misclassification penalty for all nodes.

Section \ref{sec:simulations_fasttext} uses the open source implementation of fastText \footnote{https://github.com/facebookresearch/fastText} with the minimal modifications that are needed to use a custom tree in the hierarchical softmax.
We use default fastText hyperparameters with the following changes: hierarchical softmax for the loss function, 0.5 for the learning rate, 200 for the number of epochs, 128 for the dimension of the embeddings, 2 for the minimal number of word occurrences, and 2 for the max length of word n-grams.
\vspace{-.1cm}
\subsubsection{Our clustering algorithm}

For Section \ref{sec:simulations_parabel} and Section \ref{sec:simulations_amazon}, our clustering algorithm was applied recursively until no more than 100 labels remained. The remaining labels then formed a leaf node. For Section \ref{sec:simulations_fasttext}, our clustering algorithm was applied recursively until no more than 1 label remained.

In Section \ref{sec:simulations_parabel}, we ran our algorithm for $\lambda'=$ 0, 0.5, 1, 2, 6, 10, 30, 100, 300, 1000, 100000 where $\lambda$ is obtained as $\lambda = 2 \lambda' / (1 + \lambda')$. In Section \ref{sec:simulations_fasttext}, we ran our algorithm for $\lambda'=$ 0, 0.25, 0.5, 0.75, 1, 2, 4, 16, 64, 256, 100000. In Section \ref{sec:simulations_amazon}, we ran our algorithm for $\lambda'=$ 0, 0.25, 0.5, 0.75, 1, 2, 10, 30, 1000, 100000.

In the implementation of our clustering scheme, we round the $\b{\alpha}$ after each iteration of the clustering algorithm. An additive Laplacian smoothing parameter of $\gamma=0.1$ was used for all simulations.

\section{Conclusions and Future Work}\label{sec:conc}
In this work we studied the problem of constructing probabilistic label trees that incorporate both frequency and similarity information.
While state of the art schemes ignore one of the two, yielding suboptimal statistical performance or latency, we designed a practical and efficient algorithm for generating a PLT that utilizes both label similarity and frequencies. Our scheme provides a knob to trade off between computational efficiency and statistical performance, which was not previously possible.
This approach has promising empirical performance on both public datasets and those derived from e-commerce customer logs, and provides a novel theoretical bridge between these two extremes.

One important direction of future work is realizing these expected depth gains as latency improvements in online applications.
The focus of this work was on constructing a PLT with low expected depth, meaning that the relevant labels that a standard beam search algorithm returns are high up in this tree.
Translating these expected depth gains into wall-clock improvements would require modifying the beam search procedure to use an adaptive stopping condition, returning results before finishing traversing all its paths.
Such an adaptive stopping condition would require application-specific parameters, and yield highly implementation dependent latency improvements, and so to compare against implementations of Parabel and fastText we used a standardized metric of expected depth.
Developing and analyzing this modified beam search is a critical line of future work towards realizing these expected depth gains as deployable latency improvements.

\vspace{.1cm}
\appendix
\begin{center}
    {\huge \textbf{Appendices}}
\end{center}

\section{Additional Numerical Results} \label{app:additionalSims}
Below, we include analogues of Figure \ref{fig:wiki500k_labelFreqs} for Amazon-670K (Figure \ref{fig:amazon670k_labelFreqs}) and AmazonCat-13K (Figure \ref{fig:amazoncat13k_labelFreqs}).

\begin{figure}[h]
\captionsetup[subfigure]{aboveskip=1pt}
	\centering
	\begin{subfigure}[b]{0.23\textwidth}
		\includegraphics[width=\columnwidth, trim=.3cm .5cm .7cm 1.4cm, clip]{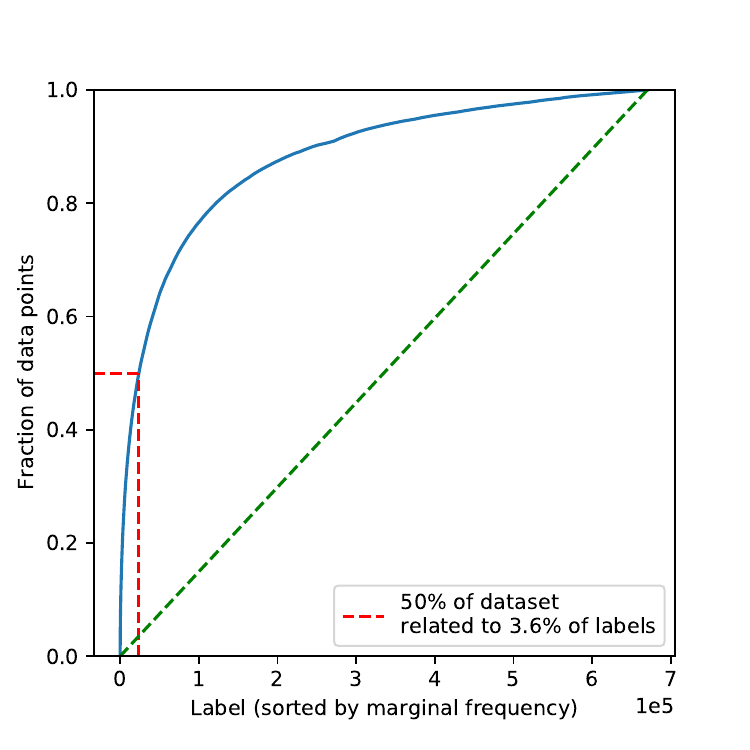}
		\caption{Amazon670k}
		\label{fig:amazon670k_labelFreqs}
		\vspace{.1cm}
	\end{subfigure}
	\hfill
	\begin{subfigure}[b]{0.23\textwidth}
		\includegraphics[width=\columnwidth, trim=.3cm .5cm .7cm 1.4cm, clip]{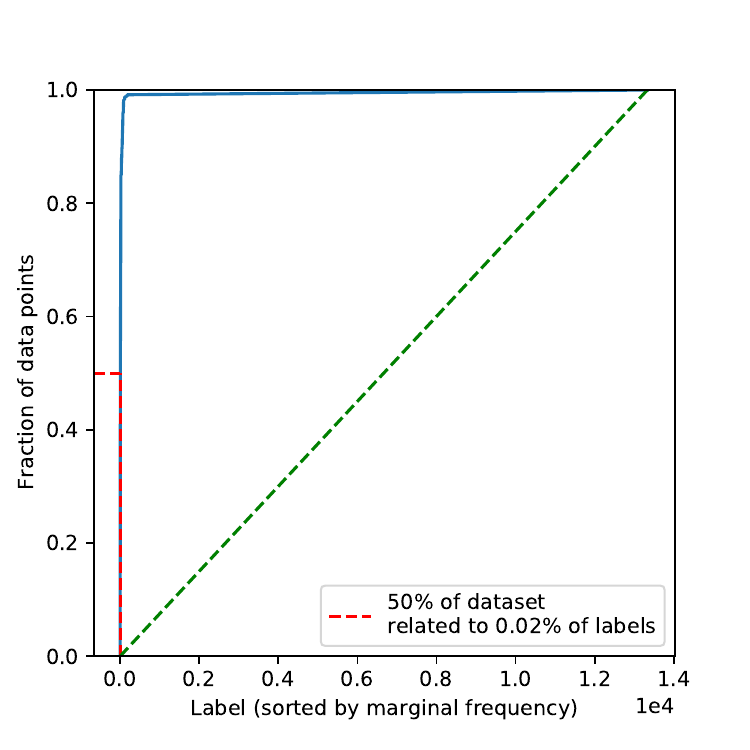}
		\caption{AmazonCat-13K}
		\label{fig:amazoncat13k_labelFreqs}
		\vspace{.1cm}
	\end{subfigure}
	\begin{subfigure}[b]{0.23\textwidth}
		\includegraphics[width=\columnwidth, trim=.3cm .5cm .7cm 1.4cm, clip]{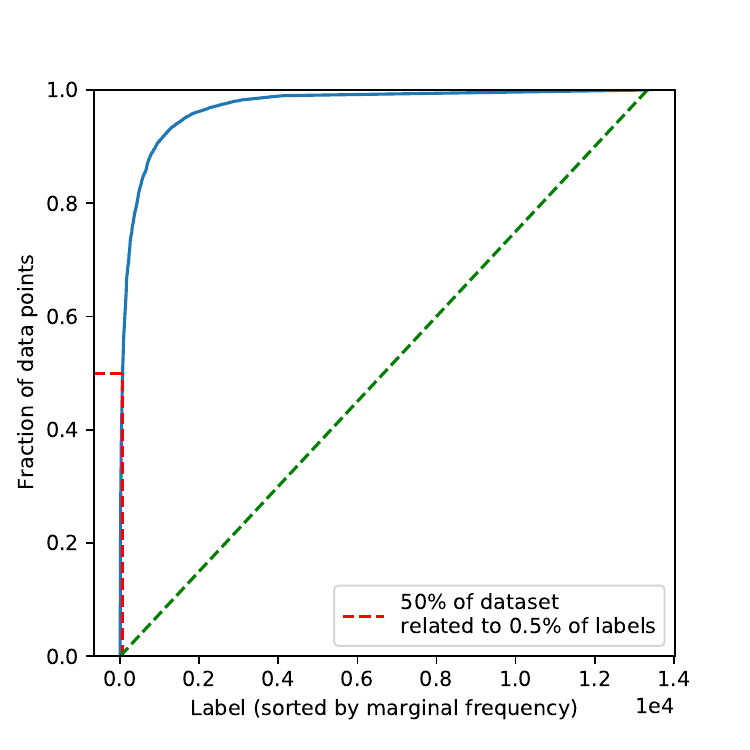}
		\caption{3 labels, AmazonCat-13K}
	\end{subfigure}
	\hfill
	\begin{subfigure}[b]{0.23\textwidth}
		\includegraphics[width=\columnwidth, trim=.3cm .5cm .7cm 1.4cm, clip]{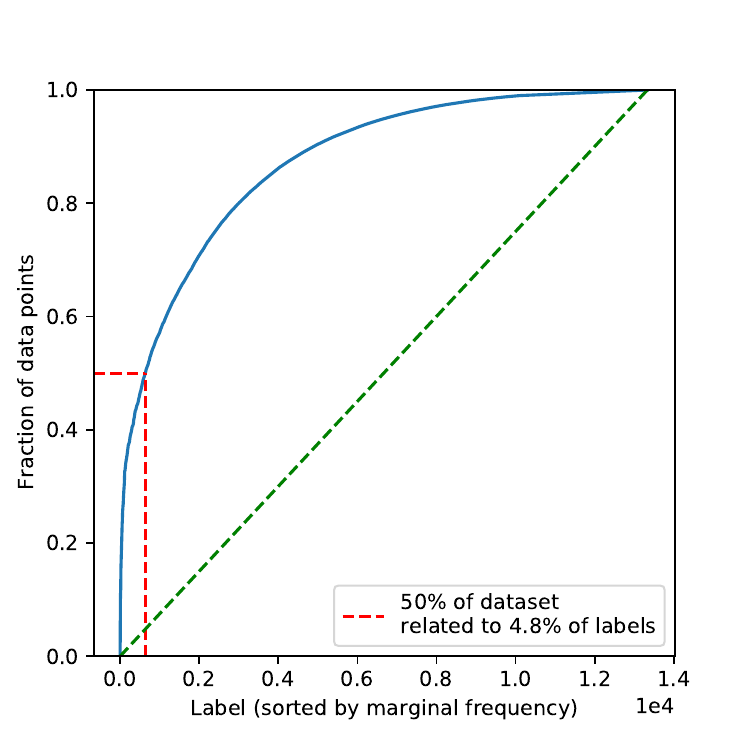}
		\caption{All labels, AmazonCat-13K}
	\end{subfigure}
\vspace{-.3cm}
	\caption{Label frequency imbalance on public datasets.}\label{fig:amazoncat13k_alllabelFreqs}
	\vspace{-.3cm}
\end{figure}

%

We can also measure the label imbalance by not just desiring one label per data point, but instead requiring 3 relevant labels per data point (to obtain 100\% recall at 3), or require that all relevant labels for a given data point be included in our subset.
We see in Figure \ref{fig:amazoncat13k_alllabelFreqs}(c,d) how this looks for AmazonCat-13K.

\section{Constructing frequency vector $\widetilde{\b f}$} \label{app:ftilde}

For the case of $k=1$, we can construct a frequency vector for our tree construction scheme by greedily finding the most frequent label in the dataset, assigning it a frequency proportional to the number of contexts it appears in, and then removing that label and its associated contexts.
This is formalized in Algorithm \ref{alg:greedyFreqConstruction}.
However, for $k>1$, this seems to be a combinatorially hard problem, for which generating an efficient solution is (to our knowledge) an open problem. 
In light of this, we construct these frequencies for $k>1$ by the same greedy scheme mentioned above, and call them $\widetilde{\b f}$ to denote them as the vector of frequencies that we use for our Fano tree.
This is to distinguish them from our vector of marginal label frequencies $\b f$, which need not be good at minimizing expected depth.
In detailing our algorithm, we utilize the label matrix $Y\in \{0,1\}^{N\times L}$, where $Y_{i,\ell}=1$ if label $\ell$ is relevant for data point $i$.

\setlength{\textfloatsep}{0pt}
\setlength{\intextsep}{3pt}

\begin{algorithm}[h]
\begin{algorithmic}[1]
\caption{\label{alg:greedyFreqConstruction} \texttt{Greedy $\widetilde{\b f}$ construction}}
\State \textbf{Input:} label matrix $Y$, context frequencies $\b p$
\State{$\widetilde{\b f} \gets \b{0}$ \Comment{initialize all 0s}}
\While{$Y^\top \b{p}\neq0$}
\State
$ i^{*}= \argmax_i Y^\top \b{p}$
\State $\widetilde{f}_{i^{*}} = [Y^\top \b{p}]_{i^{*}}$
\State $\b{p}_{Y^{(i^{*})}} = \b{0}$ \Comment{contexts that relate to label $i^{*}$}
\EndWhile
\State{\Return $\widetilde{\b f}$}
\end{algorithmic}
\end{algorithm} 

In practice, another approach for constructing $\widetilde{\b f}$ has similar performance, and a similarly intuitive justification.
We begin by taking our marginal frequency vector $\b{f} \propto Y^\top \mathds{1}$, and sorting these label frequencies by decreasing frequency to get a ranking of labels from most to least frequent.
We then iterate over the contexts, and for each context assign its frequency to the most frequent label which it contains.
This is detailed more formally in Algorithm \ref{alg:marginalFreqConstruction}.
This can be softened, as the standard vector $\b{f}$ can be constructed by, for each context, adding its frequency to each label it contains.
For this extreme $\widetilde{\b f}$ we assign all of the context's frequency just to the most common label, but we can smoothly interpolate between the two, trading off between uniformly adding the frequency to all relevant labels and just the most common.

\begin{algorithm}[h]
\begin{algorithmic}[1]
\caption{\label{alg:marginalFreqConstruction} \texttt{Marginal $\widetilde{\b f}$ construction}}
\State \textbf{Input:} label matrix $Y$, context frequencies $\b p$
\State{$\widetilde{\b f} \gets \b{0}$ \Comment{initialize all 0s}}
\State $\b{f} = Y^\top \b p$
\For{$i=1,\hdots,n$}
\State $ j^{*}= \argmax_j Y_j \odot \b{f}$ \Comment{$\odot$ is entrywise multiplication}
\State $\widetilde{f}_{j^{*}} = \widetilde{f}_{j^{*}} + p_{j^{*}}$
\EndFor
\State{\Return $\widetilde{\b f}$}
\end{algorithmic}
\end{algorithm} 
Comparing these constructions of $\widetilde{\b f}$ with $\b f$, we see that we are simply zeroing out many coordinates, and reducing the value of some others, leading to a more non-uniform distribution.

One further interesting note is that in many real world XMC applications like Dynamic Search Advertising, we not only have label frequencies, but also context frequencies.
That is, some contexts are much more common than others.
For this work we focused on a uniform distribution over contexts, but all methods we propose can be extended to this scenario with context frequencies, as denoting these with $\b{p}$ we can instead construct $\b{f} \propto Y^\top \b{p}$.

\newpage
\clearpage
 \bibliographystyle{ACM-Reference-Format}
 \balance
\bibliography{mybib.bib}


\begin{thebibliography}{31}


\ifx \showCODEN    \undefined \def \showCODEN     #1{\unskip}     \fi
\ifx \showDOI      \undefined \def \showDOI       #1{#1}\fi
\ifx \showISBNx    \undefined \def \showISBNx     #1{\unskip}     \fi
\ifx \showISBNxiii \undefined \def \showISBNxiii  #1{\unskip}     \fi
\ifx \showISSN     \undefined \def \showISSN      #1{\unskip}     \fi
\ifx \showLCCN     \undefined \def \showLCCN      #1{\unskip}     \fi
\ifx \shownote     \undefined \def \shownote      #1{#1}          \fi
\ifx \showarticletitle \undefined \def \showarticletitle #1{#1}   \fi
\ifx \showURL      \undefined \def \showURL       {\relax}        \fi
\providecommand\bibfield[2]{#2}
\providecommand\bibinfo[2]{#2}
\providecommand\natexlab[1]{#1}
\providecommand\showeprint[2][]{arXiv:#2}

\bibitem[\protect\citeauthoryear{Agrawal, Gupta, Prabhu, and Varma}{Agrawal
  et~al\mbox{.}}{2013}]%
        {agrawal2013multi}
\bibfield{author}{\bibinfo{person}{Rahul Agrawal}, \bibinfo{person}{Archit
  Gupta}, \bibinfo{person}{Yashoteja Prabhu}, {and} \bibinfo{person}{Manik
  Varma}.} \bibinfo{year}{2013}\natexlab{}.
\newblock \showarticletitle{Multi-label learning with millions of labels:
  Recommending advertiser bid phrases for web pages}. In
  \bibinfo{booktitle}{\emph{Proceedings of the 22nd international conference on
  World Wide Web}}. \bibinfo{pages}{13--24}.
\newblock


\bibitem[\protect\citeauthoryear{Babbar and Sch\"{o}lkopf}{Babbar and
  Sch\"{o}lkopf}{2017}]%
        {dismec}
\bibfield{author}{\bibinfo{person}{Rohit Babbar} {and}
  \bibinfo{person}{Bernhard Sch\"{o}lkopf}.} \bibinfo{year}{2017}\natexlab{}.
\newblock \showarticletitle{DiSMEC: Distributed Sparse Machines for Extreme
  Multi-Label Classification}. In \bibinfo{booktitle}{\emph{Proceedings of the
  Tenth ACM International Conference on Web Search and Data Mining}}
  (Cambridge, United Kingdom) \emph{(\bibinfo{series}{WSDM '17})}.
  \bibinfo{publisher}{Association for Computing Machinery},
  \bibinfo{address}{New York, NY, USA}, \bibinfo{pages}{721–729}.
\newblock
\showISBNx{9781450346757}


\bibitem[\protect\citeauthoryear{Bertoni, Goldwurm, Lin, and Sacc{\`a}}{Bertoni
  et~al\mbox{.}}{2012}]%
        {bertoni2012size}
\bibfield{author}{\bibinfo{person}{Alberto Bertoni},
  \bibinfo{person}{Massimiliano Goldwurm}, \bibinfo{person}{Jianyi Lin}, {and}
  \bibinfo{person}{Francesco Sacc{\`a}}.} \bibinfo{year}{2012}\natexlab{}.
\newblock \showarticletitle{Size constrained distance clustering: separation
  properties and some complexity results}.
\newblock \bibinfo{journal}{\emph{Fundamenta Informaticae}}
  \bibinfo{volume}{115}, \bibinfo{number}{1} (\bibinfo{year}{2012}),
  \bibinfo{pages}{125--139}.
\newblock


\bibitem[\protect\citeauthoryear{Bhatia, Dahiya, Jain, Mittal, Prabhu, and
  Varma}{Bhatia et~al\mbox{.}}{2016}]%
        {Bhatia16_XMCRepo}
\bibfield{author}{\bibinfo{person}{K. Bhatia}, \bibinfo{person}{K. Dahiya},
  \bibinfo{person}{H. Jain}, \bibinfo{person}{A. Mittal}, \bibinfo{person}{Y.
  Prabhu}, {and} \bibinfo{person}{M. Varma}.} \bibinfo{year}{2016}\natexlab{}.
\newblock \bibinfo{title}{The extreme classification repository: Multi-label
  datasets and code}.
\newblock
\newblock
\urldef\tempurl%
\url{http://manikvarma.org/downloads/XC/XMLRepository.html}
\showURL{%
\tempurl}


\bibitem[\protect\citeauthoryear{Boyd, Boyd, and Vandenberghe}{Boyd
  et~al\mbox{.}}{2004}]%
        {boyd2004convex}
\bibfield{author}{\bibinfo{person}{Stephen Boyd}, \bibinfo{person}{Stephen~P
  Boyd}, {and} \bibinfo{person}{Lieven Vandenberghe}.}
  \bibinfo{year}{2004}\natexlab{}.
\newblock \bibinfo{booktitle}{\emph{Convex optimization}}.
\newblock \bibinfo{publisher}{Cambridge university press}.
\newblock


\bibitem[\protect\citeauthoryear{Busa-Fekete, Dembczynski, Golovnev, Jasinska,
  Kuznetsov, Sviridenko, and Xu}{Busa-Fekete et~al\mbox{.}}{2019}]%
        {busa2019computational}
\bibfield{author}{\bibinfo{person}{R{\'o}bert Busa-Fekete},
  \bibinfo{person}{Krzysztof Dembczynski}, \bibinfo{person}{Alexander
  Golovnev}, \bibinfo{person}{Kalina Jasinska}, \bibinfo{person}{Mikhail
  Kuznetsov}, \bibinfo{person}{Maxim Sviridenko}, {and} \bibinfo{person}{Chao
  Xu}.} \bibinfo{year}{2019}\natexlab{}.
\newblock \showarticletitle{On the computational complexity of the
  probabilistic label tree algorithms}.
\newblock \bibinfo{journal}{\emph{arXiv preprint arXiv:1906.00294}}
  (\bibinfo{year}{2019}).
\newblock


\bibitem[\protect\citeauthoryear{Chang, Yu, Zhong, Yang, and Dhillon}{Chang
  et~al\mbox{.}}{2019}]%
        {chang2019xBert}
\bibfield{author}{\bibinfo{person}{Wei-Cheng Chang}, \bibinfo{person}{Hsiang-Fu
  Yu}, \bibinfo{person}{Kai Zhong}, \bibinfo{person}{Yiming Yang}, {and}
  \bibinfo{person}{Inderjit Dhillon}.} \bibinfo{year}{2019}\natexlab{}.
\newblock \showarticletitle{X-Bert: extreme multi-label text classification
  with using bidirectional encoder representations from transformers}.
\newblock \bibinfo{journal}{\emph{arXiv preprint arXiv:1905.02331}}
  (\bibinfo{year}{2019}).
\newblock


\bibitem[\protect\citeauthoryear{Fano}{Fano}{1949}]%
        {fano1949transmission}
\bibfield{author}{\bibinfo{person}{Robert~M Fano}.}
  \bibinfo{year}{1949}\natexlab{}.
\newblock \bibinfo{booktitle}{\emph{The transmission of information}}.
\newblock \bibinfo{publisher}{Massachusetts Institute of Technology, Research
  Laboratory of Electronics}.
\newblock


\bibitem[\protect\citeauthoryear{Huang, Wang, Medini, and Shrivastava}{Huang
  et~al\mbox{.}}{2018}]%
        {huang2018extreme}
\bibfield{author}{\bibinfo{person}{Qixuan Huang}, \bibinfo{person}{Yiqiu Wang},
  \bibinfo{person}{Tharun Medini}, {and} \bibinfo{person}{Anshumali
  Shrivastava}.} \bibinfo{year}{2018}\natexlab{}.
\newblock \showarticletitle{Extreme Classification in Log Memory}.
\newblock \bibinfo{journal}{\emph{arXiv preprint arXiv:1810.04254}}
  (\bibinfo{year}{2018}).
\newblock


\bibitem[\protect\citeauthoryear{Huffman}{Huffman}{1952}]%
        {huffman1952method}
\bibfield{author}{\bibinfo{person}{David~A Huffman}.}
  \bibinfo{year}{1952}\natexlab{}.
\newblock \showarticletitle{A method for the construction of minimum-redundancy
  codes}.
\newblock \bibinfo{journal}{\emph{Proceedings of the IRE}}
  \bibinfo{volume}{40}, \bibinfo{number}{9} (\bibinfo{year}{1952}),
  \bibinfo{pages}{1098--1101}.
\newblock


\bibitem[\protect\citeauthoryear{Jasinska, Dembczynski, Busa-Fekete,
  Pfannschmidt, Klerx, and Hullermeier}{Jasinska et~al\mbox{.}}{2016}]%
        {pmlr-v48-jasinska16}
\bibfield{author}{\bibinfo{person}{Kalina Jasinska}, \bibinfo{person}{Krzysztof
  Dembczynski}, \bibinfo{person}{Robert Busa-Fekete}, \bibinfo{person}{Karlson
  Pfannschmidt}, \bibinfo{person}{Timo Klerx}, {and} \bibinfo{person}{Eyke
  Hullermeier}.} \bibinfo{year}{2016}\natexlab{}.
\newblock \showarticletitle{Extreme F-measure Maximization using Sparse
  Probability Estimates} \emph{(\bibinfo{series}{Proceedings of Machine
  Learning Research}, Vol.~\bibinfo{volume}{48})},
  \bibfield{editor}{\bibinfo{person}{Maria~Florina Balcan} {and}
  \bibinfo{person}{Kilian~Q. Weinberger}} (Eds.). \bibinfo{publisher}{PMLR},
  \bibinfo{address}{New York, New York, USA}, \bibinfo{pages}{1435--1444}.
\newblock


\bibitem[\protect\citeauthoryear{Joulin, Grave, Bojanowski, and Mikolov}{Joulin
  et~al\mbox{.}}{2016}]%
        {joulin2016bag_fasttext}
\bibfield{author}{\bibinfo{person}{Armand Joulin}, \bibinfo{person}{Edouard
  Grave}, \bibinfo{person}{Piotr Bojanowski}, {and} \bibinfo{person}{Tomas
  Mikolov}.} \bibinfo{year}{2016}\natexlab{}.
\newblock \showarticletitle{Bag of Tricks for Efficient Text Classification}.
\newblock \bibinfo{journal}{\emph{arXiv preprint arXiv:1607.01759}}
  (\bibinfo{year}{2016}).
\newblock


\bibitem[\protect\citeauthoryear{Khandagale, Xiao, and Babbar}{Khandagale
  et~al\mbox{.}}{2020}]%
        {khandagale2020bonsai}
\bibfield{author}{\bibinfo{person}{Sujay Khandagale}, \bibinfo{person}{Han
  Xiao}, {and} \bibinfo{person}{Rohit Babbar}.}
  \bibinfo{year}{2020}\natexlab{}.
\newblock \showarticletitle{Bonsai: diverse and shallow trees for extreme
  multi-label classification}.
\newblock \bibinfo{journal}{\emph{Machine Learning}} \bibinfo{volume}{109},
  \bibinfo{number}{11} (\bibinfo{year}{2020}), \bibinfo{pages}{2099--2119}.
\newblock


\bibitem[\protect\citeauthoryear{Kraj{\v{c}}i, Liu, Mike{\v{s}}, and
  Moser}{Kraj{\v{c}}i et~al\mbox{.}}{2015}]%
        {fano_proof}
\bibfield{author}{\bibinfo{person}{Stanislav Kraj{\v{c}}i},
  \bibinfo{person}{Chin-Fu Liu}, \bibinfo{person}{Ladislav Mike{\v{s}}}, {and}
  \bibinfo{person}{Stefan~M Moser}.} \bibinfo{year}{2015}\natexlab{}.
\newblock \showarticletitle{Performance analysis of Fano coding}. In
  \bibinfo{booktitle}{\emph{2015 IEEE International Symposium on Information
  theory (ISIT)}}. IEEE, \bibinfo{pages}{1746--1750}.
\newblock


\bibitem[\protect\citeauthoryear{Liu, Chang, Wu, and Yang}{Liu
  et~al\mbox{.}}{2017}]%
        {xmlcnn}
\bibfield{author}{\bibinfo{person}{Jingzhou Liu}, \bibinfo{person}{Wei-Cheng
  Chang}, \bibinfo{person}{Yuexin Wu}, {and} \bibinfo{person}{Yiming Yang}.}
  \bibinfo{year}{2017}\natexlab{}.
\newblock \showarticletitle{Deep Learning for Extreme Multi-Label Text
  Classification}. In \bibinfo{booktitle}{\emph{Proceedings of the 40th
  International ACM SIGIR Conference on Research and Development in Information
  Retrieval}} (Shinjuku, Tokyo, Japan) \emph{(\bibinfo{series}{SIGIR '17})}.
  \bibinfo{publisher}{Association for Computing Machinery},
  \bibinfo{address}{New York, NY, USA}, \bibinfo{pages}{115–124}.
\newblock


\bibitem[\protect\citeauthoryear{McAuley and Leskovec}{McAuley and
  Leskovec}{2013}]%
        {mcauley2013hidden_amazoncat13k}
\bibfield{author}{\bibinfo{person}{Julian McAuley} {and} \bibinfo{person}{Jure
  Leskovec}.} \bibinfo{year}{2013}\natexlab{}.
\newblock \showarticletitle{Hidden factors and hidden topics: understanding
  rating dimensions with review text}. In \bibinfo{booktitle}{\emph{Proceedings
  of the 7th ACM conference on Recommender systems}}.
  \bibinfo{pages}{165--172}.
\newblock


\bibitem[\protect\citeauthoryear{Mikolov, Chen, Corrado, and Dean}{Mikolov
  et~al\mbox{.}}{2013a}]%
        {mikolov2013efficient}
\bibfield{author}{\bibinfo{person}{Tomas Mikolov}, \bibinfo{person}{Kai Chen},
  \bibinfo{person}{Greg Corrado}, {and} \bibinfo{person}{Jeffrey Dean}.}
  \bibinfo{year}{2013}\natexlab{a}.
\newblock \showarticletitle{Efficient estimation of word representations in
  vector space}.
\newblock \bibinfo{journal}{\emph{arXiv preprint arXiv:1301.3781}}
  (\bibinfo{year}{2013}).
\newblock


\bibitem[\protect\citeauthoryear{Mikolov, Sutskever, Chen, Corrado, and
  Dean}{Mikolov et~al\mbox{.}}{2013b}]%
        {mikolov2013distributed}
\bibfield{author}{\bibinfo{person}{Tomas Mikolov}, \bibinfo{person}{Ilya
  Sutskever}, \bibinfo{person}{Kai Chen}, \bibinfo{person}{Greg~S Corrado},
  {and} \bibinfo{person}{Jeff Dean}.} \bibinfo{year}{2013}\natexlab{b}.
\newblock \showarticletitle{Distributed representations of words and phrases
  and their compositionality}. In \bibinfo{booktitle}{\emph{Advances in neural
  information processing systems}}. \bibinfo{pages}{3111--3119}.
\newblock


\bibitem[\protect\citeauthoryear{Mnih and Hinton}{Mnih and Hinton}{2009}]%
        {mnih2009scalable_parabel_esque}
\bibfield{author}{\bibinfo{person}{Andriy Mnih} {and}
  \bibinfo{person}{Geoffrey~E Hinton}.} \bibinfo{year}{2009}\natexlab{}.
\newblock \showarticletitle{A scalable hierarchical distributed language
  model}. In \bibinfo{booktitle}{\emph{Advances in neural information
  processing systems}}. \bibinfo{pages}{1081--1088}.
\newblock


\bibitem[\protect\citeauthoryear{Morin and Bengio}{Morin and Bengio}{2005}]%
        {morin2005hierarchical}
\bibfield{author}{\bibinfo{person}{Frederic Morin} {and}
  \bibinfo{person}{Yoshua Bengio}.} \bibinfo{year}{2005}\natexlab{}.
\newblock \showarticletitle{Hierarchical probabilistic neural network language
  model}. In \bibinfo{booktitle}{\emph{AISTATS}}, Vol.~\bibinfo{volume}{5}.
  \bibinfo{pages}{246--252}.
\newblock


\bibitem[\protect\citeauthoryear{Partalas, Kosmopoulos, Baskiotis, Artieres,
  Paliouras, Gaussier, Androutsopoulos, Amini, and Galinari}{Partalas
  et~al\mbox{.}}{2015}]%
        {partalas2015lshtc}
\bibfield{author}{\bibinfo{person}{Ioannis Partalas}, \bibinfo{person}{Aris
  Kosmopoulos}, \bibinfo{person}{Nicolas Baskiotis}, \bibinfo{person}{Thierry
  Artieres}, \bibinfo{person}{George Paliouras}, \bibinfo{person}{Eric
  Gaussier}, \bibinfo{person}{Ion Androutsopoulos},
  \bibinfo{person}{Massih-Reza Amini}, {and} \bibinfo{person}{Patrick
  Galinari}.} \bibinfo{year}{2015}\natexlab{}.
\newblock \bibinfo{title}{LSHTC: A Benchmark for Large-Scale Text
  Classification}.
\newblock
\newblock
\showeprint[arxiv]{1503.08581}~[cs.IR]


\bibitem[\protect\citeauthoryear{Prabhu, Kag, Gopinath, Dahiya, Harsola,
  Agrawal, and Varma}{Prabhu et~al\mbox{.}}{2018a}]%
        {prabhu2018extreme}
\bibfield{author}{\bibinfo{person}{Yashoteja Prabhu}, \bibinfo{person}{Anil
  Kag}, \bibinfo{person}{Shilpa Gopinath}, \bibinfo{person}{Kunal Dahiya},
  \bibinfo{person}{Shrutendra Harsola}, \bibinfo{person}{Rahul Agrawal}, {and}
  \bibinfo{person}{Manik Varma}.} \bibinfo{year}{2018}\natexlab{a}.
\newblock \showarticletitle{Extreme multi-label learning with label features
  for warm-start tagging, ranking \& recommendation}. In
  \bibinfo{booktitle}{\emph{Proceedings of the Eleventh ACM International
  Conference on Web Search and Data Mining}}. \bibinfo{pages}{441--449}.
\newblock


\bibitem[\protect\citeauthoryear{Prabhu, Kag, Harsola, Agrawal, and
  Varma}{Prabhu et~al\mbox{.}}{2018b}]%
        {parabel}
\bibfield{author}{\bibinfo{person}{Yashoteja Prabhu}, \bibinfo{person}{Anil
  Kag}, \bibinfo{person}{Shrutendra Harsola}, \bibinfo{person}{Rahul Agrawal},
  {and} \bibinfo{person}{Manik Varma}.} \bibinfo{year}{2018}\natexlab{b}.
\newblock \showarticletitle{Parabel: Partitioned label trees for extreme
  classification with application to dynamic search advertising}. In
  \bibinfo{booktitle}{\emph{Proceedings of the 2018 World Wide Web
  Conference}}. \bibinfo{pages}{993--1002}.
\newblock


\bibitem[\protect\citeauthoryear{Prabhu and Varma}{Prabhu and Varma}{2014}]%
        {prabhu2014fastxml}
\bibfield{author}{\bibinfo{person}{Yashoteja Prabhu} {and}
  \bibinfo{person}{Manik Varma}.} \bibinfo{year}{2014}\natexlab{}.
\newblock \showarticletitle{Fastxml: A fast, accurate and stable
  tree-classifier for extreme multi-label learning}. In
  \bibinfo{booktitle}{\emph{Proceedings of the 20th ACM SIGKDD international
  conference on Knowledge discovery and data mining}}.
  \bibinfo{pages}{263--272}.
\newblock


\bibitem[\protect\citeauthoryear{Shrivastava and Li}{Shrivastava and
  Li}{2014}]%
        {shrivastava2014asymmetric}
\bibfield{author}{\bibinfo{person}{Anshumali Shrivastava} {and}
  \bibinfo{person}{Ping Li}.} \bibinfo{year}{2014}\natexlab{}.
\newblock \showarticletitle{Asymmetric LSH (ALSH) for sublinear time maximum
  inner product search (MIPS)}.
\newblock \bibinfo{journal}{\emph{arXiv preprint arXiv:1405.5869}}
  (\bibinfo{year}{2014}).
\newblock


\bibitem[\protect\citeauthoryear{Vijayanarasimhan, Shlens, Monga, and
  Yagnik}{Vijayanarasimhan et~al\mbox{.}}{2014}]%
        {vijayanarasimhan2014deep}
\bibfield{author}{\bibinfo{person}{Sudheendra Vijayanarasimhan},
  \bibinfo{person}{Jonathon Shlens}, \bibinfo{person}{Rajat Monga}, {and}
  \bibinfo{person}{Jay Yagnik}.} \bibinfo{year}{2014}\natexlab{}.
\newblock \showarticletitle{Deep networks with large output spaces}.
\newblock \bibinfo{journal}{\emph{arXiv preprint arXiv:1412.7479}}
  (\bibinfo{year}{2014}).
\newblock


\bibitem[\protect\citeauthoryear{Yang, Ruan, Li, and Hu}{Yang
  et~al\mbox{.}}{2017}]%
        {yang2017optimize_semhuff}
\bibfield{author}{\bibinfo{person}{Zhixuan Yang}, \bibinfo{person}{Chong Ruan},
  \bibinfo{person}{Caihua Li}, {and} \bibinfo{person}{Junfeng Hu}.}
  \bibinfo{year}{2017}\natexlab{}.
\newblock \showarticletitle{Optimize Hierarchical Softmax with Word Similarity
  Knowledge.}
\newblock \bibinfo{journal}{\emph{Polibits}}  \bibinfo{volume}{55}
  (\bibinfo{year}{2017}), \bibinfo{pages}{11--16}.
\newblock


\bibitem[\protect\citeauthoryear{Yen, Huang, Dai, Ravikumar, Dhillon, and
  Xing}{Yen et~al\mbox{.}}{2017}]%
        {ppdsparse}
\bibfield{author}{\bibinfo{person}{Ian~E.H. Yen}, \bibinfo{person}{Xiangru
  Huang}, \bibinfo{person}{Wei Dai}, \bibinfo{person}{Pradeep Ravikumar},
  \bibinfo{person}{Inderjit Dhillon}, {and} \bibinfo{person}{Eric Xing}.}
  \bibinfo{year}{2017}\natexlab{}.
\newblock \showarticletitle{PPDsparse: A Parallel Primal-Dual Sparse Method for
  Extreme Classification}. In \bibinfo{booktitle}{\emph{Proceedings of the 23rd
  ACM SIGKDD International Conference on Knowledge Discovery and Data Mining}}
  (Halifax, NS, Canada) \emph{(\bibinfo{series}{KDD '17})}.
  \bibinfo{publisher}{Association for Computing Machinery},
  \bibinfo{address}{New York, NY, USA}, \bibinfo{pages}{545–553}.
\newblock


\bibitem[\protect\citeauthoryear{Yen, Huang, Ravikumar, Zhong, and Dhillon}{Yen
  et~al\mbox{.}}{2016}]%
        {pdsparse}
\bibfield{author}{\bibinfo{person}{Ian En-Hsu Yen}, \bibinfo{person}{Xiangru
  Huang}, \bibinfo{person}{Pradeep Ravikumar}, \bibinfo{person}{Kai Zhong},
  {and} \bibinfo{person}{Inderjit Dhillon}.} \bibinfo{year}{2016}\natexlab{}.
\newblock \showarticletitle{PD-Sparse : A Primal and Dual Sparse Approach to
  Extreme Multiclass and Multilabel Classification}
  \emph{(\bibinfo{series}{Proceedings of Machine Learning Research},
  Vol.~\bibinfo{volume}{48})}. \bibinfo{publisher}{PMLR}, \bibinfo{address}{New
  York, New York, USA}, \bibinfo{pages}{3069--3077}.
\newblock


\bibitem[\protect\citeauthoryear{You, Zhang, Wang, Dai, Mamitsuka, and Zhu}{You
  et~al\mbox{.}}{2018}]%
        {you2018attentionxml}
\bibfield{author}{\bibinfo{person}{Ronghui You}, \bibinfo{person}{Zihan Zhang},
  \bibinfo{person}{Ziye Wang}, \bibinfo{person}{Suyang Dai},
  \bibinfo{person}{Hiroshi Mamitsuka}, {and} \bibinfo{person}{Shanfeng Zhu}.}
  \bibinfo{year}{2018}\natexlab{}.
\newblock \showarticletitle{Attentionxml: Label tree-based attention-aware deep
  model for high-performance extreme multi-label text classification}.
\newblock \bibinfo{journal}{\emph{arXiv preprint arXiv:1811.01727}}
  (\bibinfo{year}{2018}).
\newblock


\bibitem[\protect\citeauthoryear{Yu, Zhong, and Dhillon}{Yu
  et~al\mbox{.}}{2020}]%
        {yu2020pecos}
\bibfield{author}{\bibinfo{person}{Hsiang-Fu Yu}, \bibinfo{person}{Kai Zhong},
  {and} \bibinfo{person}{Inderjit~S Dhillon}.} \bibinfo{year}{2020}\natexlab{}.
\newblock \showarticletitle{PECOS: Prediction for Enormous and Correlated
  Output Spaces}.
\newblock \bibinfo{journal}{\emph{arXiv preprint arXiv:2010.05878}}
  (\bibinfo{year}{2020}).
\newblock


\end{thebibliography}
\end{document}